\documentclass[twoside]{article}

\usepackage[accepted]{aistats2022}
\usepackage[round]{natbib}

\usepackage[colorlinks,
            linkcolor=red,
            anchorcolor=blue,
            citecolor=blue
            ]{hyperref}

\usepackage{pgfplots}
\usetikzlibrary{arrows,shapes,snakes,automata,backgrounds,petri}

\usepackage{enumitem}

\usepackage[textsize=tiny]{todonotes}
\input{ucla_macro.sty}

\newcommand{\reward}{r}

\newcommand{\valueite}{\text{EVI}}

\newcommand{\state}{x}

\newcommand{\regret}{\text{Regret}}

\newcommand{\tbSigma}{\tilde{\bSigma}}
\newcommand{\tbbb}{\tilde{\bbb}}
\newcommand{\tbtheta}{\tilde{\btheta}}
\newcommand{\tbeta}{\tilde{\beta}}
\newcommand{\hbSigma}{\hat{\bSigma}}
\newcommand{\hbbb}{\hat{\bbb}}
\newcommand{\hbtheta}{\hat{\btheta}}
\newcommand{\hbeta}{\hat{\beta}}
\newcommand{\cbSigma}{\check{\bSigma}}

\newcommand{\cbtheta}{\check{\btheta}}
\newcommand{\cbeta}{\check{\beta}}

\newcommand{\barsigma}{\bar{\sigma}}

\def \algname {\text{UCRL2-VTR} }

\begin{document}

%
\runningtitle{Nearly Minimax Optimal Regret for Learning Average-reward MDPs with Linear Approximation}

%

\twocolumn[

\aistatstitle{Nearly Minimax Optimal Regret for Learning Infinite-horizon Average-reward MDPs with Linear Function Approximation}

\aistatsauthor{ Yue Wu \And Dongruo Zhou \And  Quanquan Gu }

\aistatsaddress{ UCLA  \And UCLA \And UCLA} ]

\begin{abstract}
  We study reinforcement learning in an infinite-horizon average-reward setting with linear function approximation for linear mixture Markov decision processes (MDPs), where the transition probability function of the underlying MDP admits a linear form over a feature mapping of the current state, action, and next state. We propose a new algorithm UCRL2-VTR, which can be seen as an extension of the UCRL2 algorithm with linear function approximation. We show that  UCRL2-VTR with Bernstein-type bonus can achieve a regret of $\tilde{O}(d\sqrt{DT})$, where $d$ is the dimension  of the feature mapping, $T$ is the horizon, and $D$ is the diameter of the MDP. We also prove a matching lower bound $\tilde{\Omega}(d\sqrt{DT})$, which suggests that the proposed UCRL2-VTR is  minimax optimal up to logarithmic factors. To the best of our knowledge, our algorithm is the first nearly minimax optimal RL algorithm with function approximation in the infinite-horizon average-reward setting. 
\end{abstract}

\section{INTRODUCTION}

One of the major goals of reinforcement learning (RL) is to maximize the expected accumulated reward within a certain environment, which is often represented by a \emph{Markov Decision Process (MDP)}. To achieve this goal, the agent interacts with the environment sequentially under the guidance of certain \emph{policy}, receives the reward returned by the environment, and updates the policy. There are several MDP settings such as finite-horizon episodic MDPs, infinite-horizon discounted MDPs, and infinite-horizon average-reward MDPs (See \citet{sutton2018reinforcement} for a detailed introduction). Among them, the infinite-horizon average-reward MDP attracts a lot of attention, because it not only imposes the least constraints on the underlying MDP structure but also serves as a proper environment for many real-world decision-making problems that care more about the long-term return, such as factory optimization, 
and product delivery~\citep{proper2006scaling}, and automated trading in the financial markets. This is also the focus of this work. 

A series of previous work \citep{bartlett2009regal,jaksch2010near,fruit2018efficient, talebi2018variance, zhang2019regret} has proved both upper and lower regret bounds for learning infinite-horizon average-reward MDP in the tabular setting, where the number of states and actions are finite. Specifically, \citet{jaksch2010near} first  proposed UCRL2 algorithm which achieves $\tilde{O}(D S \sqrt{AT})$, where $S$ is the number of states, $A$ is the number of actions, and $D$ is the diameter of MDP. \citet{jaksch2010near} also proved that no algorithm can achieve a regret bound lower than ${\Omega}(\sqrt{DSAT})$. 
\citet{zhang2019regret} proposed a nearly minimax optimal algorithm EBF with $\tilde{O}( \sqrt{DSAT})$ regret, which matches the lower regret up to logarithmic factors. 
However, these regret bounds depend on the cardinalities of state and action spaces (i.e., $S$ and $A$), which prevent the application of these algorithms to real-world RL with large state and action spaces.  

To overcome the curse of large state space, \emph{function approximation} has been used to design practically successful algorithms \citep{singh1995reinforcement, mnih2015human, bertsekas2018feature}. 
However, most existing studies on learning infinite-horizon average-reward MDPs are limited to tabular MDPs, with only a few exceptions \citep{abbasi2019politex, abbasi2019exploration, hao2020provably, wei2020learning}. 
More specifically, \citet{abbasi2019politex, abbasi2019exploration, hao2020provably} studied RL with function approximation for infinite-horizon average-reward MDPs under strong assumptions such as uniformly-mixing and uniformly
excited feature, and proved sublinear regrets. 
More recently, \citet{wei2020learning} considered the linear MDP \citep{yang2019sample,jin2019provably} where the transition probability and reward function can be represented as linear functions over given feature mapping defined on the state and action pair, and proposed two algorithms, FOPO and OLSVI.FH. FOPO achieves $\tilde O(\sqrt{d^3T})$ regret but relies on solving a fixed-point equation at each iteration, which is computationally inefficient. OLSVI.FH, on the other hand, is computationally efficient, but can only achieve $\tilde{O}(T^{3/4})$ regret.
\citet{wei2021learning} also proposed another algorithm MDP-EXP2, which relies on very strong assumptions including the uniform-mixing and uniformly-excited-features assumptions to achieve $\tilde{O}(\sqrt{T})$ regret. 
It remains an open question that if and how the minimax optimality of learning infinite-horizon average-reward MDPs can be achieved with linear function approximation.

  \textbf{Main contribution.} 
In this paper, we resolve the above open question by 
proving nearly matching upper and lower regret bounds for a class of infinite-horizon MDPs called \textit{linear mixture/kernel MDP}  \citep{jia2020model,ayoub2020model,zhou2020provably} in the average-reward setting. Specifically, we propose a \algname algorithm based on the principle of  ``Optimism-in-Face-of-Uncertainty  (OFU)''. At the core of our algorithm is a variant of value-targeted regression \citep{jia2020model} which estimates the unknown transition probability by least-squares over the expectations of bias functions, along with an optimistic exploration.
We consider both Hoeffding-type bonus and Bernstein-type bonus for exploration.
 We show that \algname with Hoeffding-type bonus achieves a regret bound of $\tilde{O}(dD \sqrt{T})$, and \algname with Bernstein-type bonus can improve the regret to be $\tilde{O}(d \sqrt{DT})$, where $T$ is the time horizon, $D$ is the diameter of the MDP, and $d$ is the dimension of the feature mapping. 
 We also prove a $\tilde{\Omega}(d\sqrt{DT})$ lower bound on the regret of any algorithms under a given linear mixture MDP. 
 The improved upper bound and the lower bound match each other up to the logarithmic factors. To the best of our knowledge, this is the first RL algorithm with linear function approximation that achieves a nearly-minimax optimal regret bound under the infinite-horizon average-reward MDP setting.
 
\textbf{Technical novelty.} 
Compared with many recent works on RL with linear function approximation \citep{abbasi2011improved, azar2017minimax, jin2019provably, ayoub2020model, zhou2020nearly} which study either the episodic MDP or discounted MDP settings, our work focuses on a quite different setting called the average-reward setting. Therefore, most algorithms and analyses are not directly applicable or extendable to our setting. For example, the average-reward setting uses a different notion of regret compared with those in the episodic setting or discounted settings. Thus, to bound the regret, a new regret decomposition is required, which does not appear in the aforementioned works.
Besides the new regret decomposition, we also develop some other novel techniques.
To bound the regret of the Bernstein-type algorithm, we construct a different variance estimator based on the centered value function $w_k$, instead of the standard value function. 
Another technical contribution we make is that we prove a new, non-trivial total variance lemma for the average-reward setting with linear function approximation, 
which plays a pivotal role in achieving minimax optimality.


\noindent\textbf{Notation.} We use lower case letters to denote scalars,  lower and upper case bold letters to denote vectors and matrices. We use $\| \cdot \|$ to indicate Euclidean norm, and for a semi-positive definite matrix $\bSigma$ and any vector $\xb$, $\| \xb \|_{\bSigma} := \| \bSigma^{1/2} \xb \| = \sqrt{\xb^{\top} \bSigma \xb}$. For a real value $x$ and an interval $[a,b]$, we use $[x]_{[a,b]}$ to indicate the projected value from $x$ onto $[a,b]$.
We also use the standard $O$ and $\Omega$ notations. We say $a_n = O(b_n)$ if and only if $\exists C > 0, N > 0, \forall n > N, a_n \le C b_n$; $a_n = \Omega(b_n)$ if $a_n \ge C b_n$. The notation $\tilde{O}$ is used to hide logarithmic factors. For a random variable $X$, we say $X$ is $R$-sub-Gaussian if $\EE[X] = 0$ and for any $s \in \RR$, $\EE[e^{sX}] \leq e^{R^2s^2/2}$. 

\section{RELATED WORK} \label{sec:related-work}
\textbf{Infinite-horizon average-reward tabular MDPs.}
Learning infinite-horizon average-reward tabular MDPs has been thoroughly studied in the past decade. \citet{bartlett2009regal, jaksch2010near} are among the first works providing $\Tilde{\cO}(\sqrt{T})$ regret. 
Recently, many algorithms were proposed to provide tighter regret bounds under various assumptions.
\citet{ouyang2017learning} proposed a PSRL that achieves the same regret as \citet{jaksch2010near}.
\citet{agrawal2017optimistic} proposed an algorithm using Thompson-sampling with an $\tilde{\cO}(D\sqrt{SAT})$ regret.  
 \citet{fruit2018efficient} proposed a SCAL algorithm with an $\tilde{\cO}(\text{sp}(h^*)\sqrt{\Gamma SAT})$ regret, where $\text{sp}$ is the span operator , $h^*$ is the optimal bias function for weakly-communicating MDP and $\Gamma$ is the number of next states. \citet{fruit2018near} proposed a TUCRL algorithm for weakly-communicating MDPs. 
 \citet{talebi2018variance} proposed a KL-UCRL algorithm with an $\tilde{\cO}(\sqrt{S\VV T})$ regret where $\VV$ is the summation of variances of the bias function with respect to all state-action pairs. \citet{zhang2019regret} proposed an EBF algorithm with the near-optimal $\tilde{\cO}(\sqrt{\text{sp}(h^*) SAT })$ regret. \citet{fruit2020improved} proposed a UCRL2B algorithm with an $\tilde{\cO}(\Gamma\sqrt{SDAT })$ regret. \citet{ortner2020regret} proposed an OSP algorithm for ergodic MDPs with an $\tilde{\cO}(\sqrt{t_{\text{mix}}SAT })$ regret, where $t_{\text{mix}}$ is the mixing time parameter. \citet{wei2020learning} proposed two model-free algorithms: optimistic Q-learning algorithm with an $\tilde{\cO}(\text{sp}(h^*)(SA)^{1/3}T^{2/3})$ regret, and MDP-OOMD for ergodic MDPs with an $\tilde{\cO}(\sqrt{t_{\text{mix}}^3 \rho AT})$ regret, where $\rho$ is some distribution mismatch coefficient. Our algorithm is inspired by the algorithm UCRL2 proposed by~\citet{jaksch2010near}, which maintains a confidence set for the transition model and uses Extended Value Iteration (EVI) to obtain an optimistic model and a near-optimal policy under this model. Our work is also closely related to \citet{fruit2020improved}, which employs the empirical Bernstein inequality to provide tighter regret bound. 

\noindent \textbf{RL with linear function approximation.}
Provable RL with linear function approximation has received increasing interest in recent years. \citet{jiang2017contextual} proposed the low Bellman rank assumption and designed an OLIVE algorithm that achieves low sample complexity. As a special case of low Bellman rank MDPs, the \textit{linear MDP} class~\citep{yang2019sample,jin2019provably} assumes the transition probability and the reward function are linear functions 
Another related class of MDPs is the \textit{linear mixture MDP}~\citep{jia2020model, ayoub2020model, zhou2020provably}, which assumes the transition probability is a linear combination of the feature mappings over the state-action-next-state triplet. Representative works include \citet{yang2019reinforcement, modi2019sample, jia2020model,zhou2020provably, cai2020provably,he2021nearly}. Note that while linear MDPs and linear mixture MDPs share some common subset, one cannot be covered by the other \citep{zhou2020provably}.
Our algorithm and upper bound results are for linear mixture MDPs, while our lower bound holds for both MDPs. It is worth noting that for inhomogeneous episodic linear mixture MDPs, \citet{zhou2020nearly} obtained matching upper and lower bounds of regret respectively, i.e., $\tilde O(dH\sqrt{T})$ and $\Omega(dH\sqrt{T})$, where $H$ is the length of the episode. In comparison, we provide the first matching upper and lower bounds of regret in the infinite-horizon average-reward setting for linear mixture MDPs.

\section{PRELIMINARIES} \label{sec:preliminary}

We denote a Markov Decision Process  (MDP) by a tuple $M(\cS, \cA, \reward, \PP)$, where $\cS$ is the state space, $\cA$ is the action space, $\PP(s'|s,a)$ is the transition probability function, and $\reward: \cS \times \cA \rightarrow [0,1]$ is the reward function. A deterministic stationary policy $\pi: \cS \rightarrow \cA$ maps a given state $s$ to a certain action $a$. In this work, we assume that the numbers of states and actions are finite, i.e., $|\cS|, |\cA|<\infty$ and the reward $\reward$ is known and deterministic. An algorithm $A$, which starts from some 
initial state $s \in \cS$, will  (stochastically) decide at each round $t$ what action $a_t$ to take. Given the state space $\cS$, the action space $\cA$, the algorithm starting from $s_0 = s$ actually induces a Markov process $\{ s_t,a_t \}_t$, and it is natural to define the \textit{undiscounted} reward as:
\begin{align*}
    R(M, A, s, T) := \textstyle{\sum_{t=1}^{T}} r(s_t, a_t) = \textstyle{\sum_{t=1}^{T}} r_t.
\end{align*}
We can define the expected average reward over time $T$ as $\EE[R(M,A,s,T)] / T$ , and its limit
\begin{align*}
    \rho(M,A,s)
    & = 
    \lim_{T \rightarrow \infty}
    \EE[R(M,A,s,T)]/T
    ,
\end{align*}
is called the \textit{average reward}.

In general, the learnability of an MDP depends on its transition structure.  Following \cite{jaksch2010near}, we define the \textit{diameter} of an MDP as the expected steps taken from one state $s$ to another state $s'$ under the fastest stationary policy. Throughout the paper, the diameter is only used to provide an upper bound on the span of the value function.
\begin{definition}
Let $T(s'|M,\pi,s)$ be the random variable for the number of steps after which the state $s'$ is reached for the first time starting from $s$, under the MDP $M$ and policy $\pi$. Then the diameter $D(M)$ is defined as $D(M) = 
    \max_{s, s'}
    \min_{\pi: \cS \rightarrow \cA}
    \EE 
    [
    T(s'|M,\pi,s)
    ]$.
\end{definition}
In this work, following \cite{jaksch2010near}, we consider \emph{communicating} MDPs \citep{puterman2014markov} which has a \emph{finite} diameter. We assume a \emph{known} upper bound $D$ on the diameters of all MDPs considered. One may question whether assuming a known upper bound is too restrictive. This will be further explained later in Remark~\ref{remark:unknown-D}. 

For MDPs with finite diameter, the optimal average reward does not depend on the starting state. Therefore, we can define
\begin{align*}
    \rho^*(M)
    & =
    \rho^*(M,s)
    :=
    \max_{\pi} \rho(M,\pi, s),
\end{align*}
where $\pi$ is any stationary policy. Naturally, the regret is defined as \begin{align*}
    \text{Regret}(M, A, s, T)
    & := 
    T \rho^*(M) - R(M, A, s, T).
\end{align*}
Besides the finite diameter assumption, another widely used assumption for infinite-horizon average-reward MDPs is the \emph{finite span of optimal bias function (finite span)} assumption \citep{bartlett2009regal, zhang2019regret, wei2020learning}, which assumes that there exists a $\rho^* \in \RR$, $h^*: \cS \times \cA \rightarrow \RR$ such that for any $s \in \cS , a \in \cA$, the following Bellman optimality equation holds:
\begin{align}
    \rho^* + h^*(s,a) = r(s,a) + \EE_{s' \sim \PP(\cdot|s,a)}[\max_{a' \in \cA} h^*(s', a')],\notag
\end{align}
where the span of $\max_{a' \in \cA} h^*(s, a')$ defined as $\max_{s,s' \in \cS}|\max_{a' \in \cA} h^*(s, a') - \max_{a' \in \cA} h^*(s', a')|$ is finite, $h^*$ is the optimal state-action bias function. It can be verified that for any MDP with finite diameter, the span of the optimal bias function is also finite. 
We leave it as future work to extend our algorithm to deal with the finite span assumption. 

\noindent\textbf{Linear mixture MDPs.} In many application scenarios where the state space $\cS$ and action space $\cA$ are intractably large, a certain structure of the transition kernel $\PP(\cdot|s,a)$ will still enable efficient learning. One of such structures is 
\textit{linear mixture MDP} \citep{modi2019sample,zhou2020provably,jia2020model,ayoub2020model}, where the transition kernel can be represented by a linear combination of feature mappings.
\begin{definition}\label{assumption-linear}
$M(\cS, \cA, \reward, \PP)$ is called a $B$-bounded linear mixture MDP if there exist a \emph{known} feature mapping $\bphi(s'|s,a): \cS \times \cA \times \cS \rightarrow \RR^d$ and an \emph{unknown} vector $\btheta^* \in \RR^d$ with $\|\btheta^*\|_2 \leq B$, such that
\begin{itemize}[leftmargin = *]
    \item For any state-action-state triplet $(s,a,s') \in \cS \times \cA \times \cS$, we have $\PP(s'|s,a) = \la \bphi(s'|s,a), \btheta^*\ra$;
    \item For any bounded function $F: \cS \rightarrow [0,1]$ and any tuple $(s,a)\in \cS \times \cA$, we have $\|\bphi_{F}(s,a)\|_2 \leq 1$, where $\bphi_{F}(s,a) = \sum_{s'}\bphi(s'|s,a)F(s') \in \RR^d$.\label{def:bphi}
\end{itemize}

We denote the linear mixture MDP by $M_{\btheta^*}$ for simplicity.
\end{definition}
The motivation behind defining $\bphi_{F}$ is as follows: for any state action pair $s,a$, the expectation of $F(s')$ is
\begin{align*}
    [\PP F](s,a) 
    : & = 
    \EE_{s'\sim \PP}[F(s')]
    \\
    & 
    = 
    \textstyle{\sum_{s' \in \cS}} \PP(s'|s,a) F(s')=
    \la \btheta^*,
    \bphi_F(s,a)
    \ra ,
\end{align*}
which is a linear function of $\bphi_F(s,a)$. We will use this fact in both algorithm design and analysis.

It is noteworthy that tabular MDPs can be covered by linear mixture MDPs via setting the feature mapping $\bphi(s'|s,a)$ to be a one-hot vector for $s,a,s'$.

A recent work \citep{wei2020learning} studied RL with linear function approximation under the \emph{linear MDP} assumption, which assumes that there exists  a known feature mapping $\bPhi(s,a)$ and an unknown mapping $\bmu(s')$ such that $\PP(s'| s,a) = \la \bPhi(s,a), \bmu(s') \ra$. Both linear mixture MDPs and linear MDPs cover some common MDP, including tabular MDPs and bilinear MDPs \citep{yang2019reinforcement}. However, in general, they are two different classes of MDPs, because their feature mappings are defined over different domains \citep{zhou2020provably}, and neither can cover the other.

In the rest of this paper, we also use the following shorthand to indicate the variance of the random variable $F(s')$ under distribution $P(\cdot |s,a)$:
\begin{align*}
    [\VV_{P} F](s,a) 
    &
    : = 
    \EE_{s'\sim P}[F^2(s')]
    -
    \EE_{s'\sim P}[F(s')]^2,
\end{align*}
and $[\VV F]$ denotes the case when $P$ is the true transition probability $\PP$ of the MDP. 

\begin{algorithm*}[h]
	\caption{Upper-Confidence Reinforcement Learning with Value Targeted Regression ($\algname$)}\label{algorithm}
	\begin{algorithmic}[1]
	\REQUIRE Regularization parameter $\lambda$, upper bound $B$ of $\| \btheta^* \|_2$, precision of extended value iteration rounds $\epsilon$
	\STATE	Receive $s_1$, Set $k \leftarrow 0$, $t_0 \leftarrow 1$
	\STATE \textbf{OPTION 1 (Hoeffding-type Bonus):} Set $\hbSigma_1 \leftarrow \lambda\Ib$, $\hbbb_0 \leftarrow \zero$ 
	\\
	\textbf{OPTION 2 (Bernstein-type Bonus):} Set $\hbSigma_1, \tbSigma_1 \leftarrow \lambda\Ib$, $\hbbb_1, \tbbb_1, \hbtheta_1, \tbtheta_1 \leftarrow \zero$
	\STATE  Set $\pi_0(\cdot | s) \leftarrow \text{uniform}(\cA), \forall s \in \cS$ 
	\FOR{$t= 1, 2, \ldots$}
	\IF {$\text{det}(\hbSigma_{t}) \le 2\text{det}(\hbSigma_{t_k})$} \label{line5}
	\STATE $\pi_t \leftarrow \pi_{t-1}$ \{ Keep the policy unchanged\} \label{line6}
	\ELSE 
	\STATE $k \leftarrow k+1$, $t_k \leftarrow t$ \{ Starting a new episode $k$\} \label{line8}
	\STATE 
	Set $\hat{\cC}_t $ as \eqref{eqn:hat-C} (\textbf{OPTION 1}) or \eqref{eqn:hat-C-B} (\textbf{OPTION 2}) \label{line9}
	\\
	\STATE 
	Set $u_{k}(s) \leftarrow \valueite(\hat{\cC}_t, \epsilon)$
	\STATE Denote $w_{k}(s) = u_{k}(s) - \big(\max u_k(\cdot) - \min u_k(\cdot) \big)/2$
	\label{line11}
	\STATE Set $\pi_{t}(s)$ as~\eqref{eqn:EVI-output}
	\{ Compute new policy\} \label{line12}
	\ENDIF
	\STATE Take action $a_t  = \pi_t(s_t)$, receive $s_{t+1} \sim \PP(\cdot| s_t, a_t)$ \label{line14}
	\STATE \textbf{OPTION 1 (Hoeffding-type Bonus): }
	\STATE \qquad Set $\hbSigma_{t+1} \leftarrow \hbSigma_{t} + \bphi_{{w}_k}(s_t,a_t)\bphi_{{w}_k}(s_t,a_t)^\top$, $\hbbb_{t+1} \leftarrow \hbbb_{t} + \bphi_{w_k}(s_t, a_t)w_k(s_{t+1})$
	\label{line16}
	\STATE \textbf{OPTION 2 (Bernstein-type Bonus):}
	\STATE \qquad Set $[\bar{\VV}_t w_k]$ as  in~\eqref{eqn:estimated-variance} and $E_t$ as in~\eqref{eqn: variance-upper-bound} \label{line18}
	\STATE \qquad Set $\bar{\sigma}_t \leftarrow \sqrt{ \max \{ D^2/d, [\bar{\VV}_t w_k](s_t,a_t) + E_t \}}$ \label{line19}
	\STATE \qquad Set $\hbSigma_{t+1} \leftarrow \hbSigma_{t} + \bar{\sigma}_t^{-2} \bphi_{{w}_k}(s_t,a_t)\bphi_{{w}_k}(s_t,a_t)^\top$,
	$\hbbb_{t+1} \leftarrow \hbbb_{t} + \bar{\sigma}_t^{-2} w_k(s_{t+1}) \bphi_{w_k}(s_t, a_t)$ \label{line20}
	\STATE \qquad Set $\tbSigma_{t+1} \leftarrow \tbSigma_{t} +  \bphi_{{w}^2_k}(s_t,a_t)\bphi_{{w}^2_k}(s_t,a_t)^\top$,
	$\tbbb_{t+1} \leftarrow \tbbb_{t} + w^2_k(s_{t+1}) \bphi_{w^2_k}(s_t, a_t)$
    \STATE  \qquad Set 
    $\tbtheta_{t+1} \leftarrow \tbSigma_{t+1}^{-1} \tbbb_{t+1}$\label{line22}
    \STATE Set $\hbtheta_{t+1} \leftarrow \hbSigma_{t+1}^{-1} \hbbb_{t+1}$ \label{line23}
	\ENDFOR
	\end{algorithmic}
\end{algorithm*}

\begin{algorithm}[h]
	\caption{Extended Value Iteration (EVI)} \label{alg:EVI}
	\begin{algorithmic}[1]
	\REQUIRE A set $\cC$, a desired accuracy level $\epsilon$ 
	\STATE Set $u^{(0)}(s) \leftarrow 0, \forall s \in \cS$ 
	\STATE Set $i \leftarrow 0$
    \REPEAT
    \STATE $\forall s \in \cS$, set $u^{(i+1)}(s)\leftarrow \max_{a \in \cA} 
        \big\{ 
        r(s,a)
        + 
        \max_{\btheta \in \cC \cap \cB}
        \{ 
           \la 
          \btheta,  \bphi_{u^{(i)}}(s,a)
           \ra
        \} 
        \big\}$\label{EVI-line4}
    \UNTIL{ 
    $\max_{s \in \cS} \{u^{(i+1)}(s) - u^{(i)}(s)\}
    -
    \min_{s \in \cS} \{u^{(i+1)}(s) - u^{(i)}(s)\} \le \epsilon$
    } \label{EVI-line5}
    \STATE Return $u^{(i)}(s)$
	\end{algorithmic}
\end{algorithm}

\section{ALGORITHMS} \label{sec:algorithm}
We are going to present two algorithms. The first algorithm is $\algname$  (Algorithm \ref{algorithm}), which extends the UCRL2 algorithm by \cite{jaksch2010near} from tabular MDPs to linear mixture MDPs. $\algname$ includes two types of exploration strategies: Hoeffding-type bonus (OPTION I), and Bernstein-type bonus  (OPTION II). The main difference between $\algname$ and UCRL2 is the construction of confidence sets, due to the difference between the tabular MDP and the linear mixture MDP. The second algorithm is extended value iteration (EVI) (Algorithm \ref{alg:EVI}), which serves as a subroutine of $\algname$ to calculate the optimistic estimation of the value function. 


\subsection{$\algname$ with Hoeffding-type bonus}\label{sec:hoff}
We first present UCRL2-VTR with the \textbf{Hoeffding-type bonus (OPTION I)}. 
The learning process can be divided into several episodes indexed by $k$ ($t_k \le t < t_{k+1}$). At the beginning of each episode, we call the subroutine Extended Value Iteration (EVI) \citep{jaksch2010near} under a given confidence set $\hat{\cC}_{t_k}$, which contains the true parameter $\btheta^*$ with high probability. Within each episode, we follow the induced  policy $\pi_{t_k}$ and use the new 
observation to obtain a better confidence set.

\noindent\textbf{How EVI works?} In Line~\ref{EVI-line4} of Algorithm~\ref{alg:EVI}, it performs one step of extended value iteration, which takes maximum over both the action set $\cA$ and the set of plausible transition models $\cC\cap\cB$. EVI terminates when the difference between two consecutive iterations is small enough (Line~\ref{EVI-line5}).

Generally speaking, EVI outputs the optimal value function corresponding to a near-optimal MDP among all plausible MDPs contained in the confidence set $\hat\cC_{t_k}\cap\cB$, which is similar to its counterpart for tabular MDPs in \cite{jaksch2010near}. The main difference from the EVI for tabular MDPs is that we need to restrict the parameter in the set $\cB$ as well, which admits all parameters $\btheta$ that can induce a transition probability, i.e.,
$\cB= \cap_{s,a}\cB_{s,a}$ with  $\cB_{s,a}=\big\{\btheta:\la\bphi(\cdot | s,a),\btheta \ra \text{is a probability function}\big\}$.
It is easy to show that $\cB$ is a convex set. For some special feature mapping $\bphi$, $\cB$ can be a very simple convex set. For instance, when $\bphi$ is a collection of $d$ transition probability functions, $\cB$ is the $d$-dimension simplex \citep{modi2019sample}. This will make the optimization in each step of EVI easy to solve.
In detail, given the accuracy parameter $\epsilon$, $\text{EVI}(\hat{\cC}_t, \epsilon)$ outputs a value function $u^{(i)}$ satisfying 
\begin{align}
    |
    u^{(i+1)}(s)
    -
    u^{(i)}(s)
    -
    \rho_k
    |
    \le 
    \epsilon,\label{eq:uuu}
\end{align}
where $\rho_k$ is the average reward under $\PP_k$ and $\pi_{t_k}$, both of which are defined as follows
\begin{align}\label{eqn:EVI-output}
    \PP_k(\cdot | s,a)  &:=
    \la \btheta_k(s,a),
    \bphi(\cdot | s,a) 
    \ra, 
    \\
    \btheta_k(s,a) &:= \argmax_{\btheta \in \hat{\cC}_{t_k} \cap \cB}  
          \big \la 
          \btheta,  \bphi_{u_k}(s,a)
          \big \ra
     \notag \\
    \pi_{t_k}(s) &:= \argmax_{a \in \cA} 
        \big\{ 
        r(s,a)
        + 
          \big \la 
          \btheta_k(s,a),  \bphi_{u^{(i)}}(s, a)
          \big \ra
        \big\}.\notag
\end{align}
When the context is clear, we will abuse the notation a little bit and use $\btheta_k$ to denote $\btheta_k(s_t,a_t)$ for different $t$. It is worth noting that $u(\cdot)$ is quite different from the traditional value function $V(\cdot)$ defined for episodic MDPs or discounted infinite-horizon MDPs. This requires a different analysis in the later proof. 

Then the agent uses the greedy policy $\pi_{t_k}$ with respect to $u^{(i)}$ to select the actions in $k$-th episode. It is easy to see that centering $u^{(i)}(s)$ to $u^{(i)}(s) - (\max u^{(i)}(\cdot) - \min u^{(i)}(\cdot))/2$ does not change $\pi_{t_k}$. Therefore, we consider $w_k$ (in Line~\ref{line11}) as the centered version of $u_k$ in the later analysis.

An important observation for our later analysis, made by~\citet{jaksch2010near}, is that $u^{(i)}(s)$ computed in EVI (Algorithm~\ref{alg:EVI}) satisfies
\begin{align}\label{eq:gu0001}
    \max_{s \in \cS}u^{(i)}(s)
    -
    \min_{s \in \cS}u^{(i)}(s)
    \le D.
\end{align}
This is because $u^{(i)}(s)$ is the expected total $i$-step reward of an optimal non-stationary $i$-step policy starting from $s$. Suppose for $s$ and $s'$ we have $u^{(i)}(s) - D > u^{(i)}(s')$, then we can obtain a better $u^{(i)}(s')$ by adopting the following policy: first travel to $s$ as fast as possible (which takes at most $D$ steps in expectation), then following the optimal policy for $s$. Since the reward for each step belongs to $[0,1]$, the new policy will gain at least $u^{(i)}(s) - D$, contradicting the optimality of $u^{(i)}(s')$. 
By \eqref{eq:gu0001}, we also have the centered version $w_k$ satisfy $|w_k(s)| \le D/2$. It is safe and reasonable to consider $w_k$ instead of $u_k$ since we only care $\argmax u^{(i+1)}(\cdot)$ as the greedy policy. 

\noindent\textbf{Convergence and efficiency of EVI.} 
Here we briefly discuss the convergence and computational efficiency of EVI. 
For the convergence, as we will show in the next section, the set of plausible MDPs induced by $\btheta \in \cC \cap \cB$ includes the true MDP $M_{\btheta^*}$ with high probability. Since we assume $M_{\btheta^*}$ is communicating in this work, according to Theorem 7 in \cite{jaksch2010near}, EVI is guaranteed to converge. Here we present a sufficient condition under which EVI converges within logarithmic number of iterations. Define the quantity $\gamma(\cC)$ as follows:
\begin{align*}
    \max_{\btheta \in \cC, s,s' \in \cS, a,a' \in \cA} 
    \bigg[ 
    1
    -
    \sum_{j \in \cS} 
    \min 
    \big\{
    \PP_{\btheta}(j | s,a)
    ,
    \PP_{\btheta}(j | s',a')
    \big\}
    \bigg],
\end{align*}
and it is shown by \citet{puterman2014markov} (see Theorem 6.6.6.) that $\gamma(\cC)$ serves as a contraction coefficient:
\begin{align*}
    \text{span}
    (u^{(i+1)} - u^{(i)})
    & \le 
    \gamma(\cC) \cdot
    \text{span}
    (u^{(i)} - u^{(i-1)}).
\end{align*}
Therefore, as long as $\gamma(\cC) < 1$, EVI will converge within logarithmic steps.

For the computation, suppose additionally, the feature mapping admits the form $[\bphi(s'|s,a)]_j = [\bpsi(s')]_j \cdot [\bmu(s,a)]_j$ \citep{yang2019reinforcement}, then we can use Monte Carlo integration to avoid the summation over the whole state space. In fact, we only need a few evaluations on each $[\bpsi(s')]_j$ to obtain an accurate enough estimator $\hat{\bphi}(s'|s,a)$ and perform the maximization over the estimated integration. This kind of feature mapping gives rise to a special case of linear mixture MDPs, namely bilinear MDPs \citep{yang2019reinforcement}. Note that bilinear MDPs belong to both linear mixture MDPs and linear MDPs \citep{jin2019provably}. More detailed discussions on the computational complexity of EVI can be found in \cite{zhou2020provably}.

\noindent\textbf{Construction of $\hat\cC_t$.}
Now we discuss how to construct the confidence set $\hat\cC_t$ at the end of each episode. We construct $\hat\cC_t$ as an ellipsoid centering at $\hat\btheta_t$ with covariance matrix $\hat\bSigma_t$ defined in Line~\ref{line16} of Algorithm~\ref{algorithm}. Moreover, we construct $\hat\btheta_t$ as the minimizer to the ridge regression problem over contexts $\bphi_{w_k}(s_t, a_t)$ and targets $w_k(s_{t+1})$ with regularizer $\lambda\|\btheta\|_2^2$, whose closed-form solution is given in Line~\ref{line23}. The reason why we construct such a $\hat\btheta_t$ is due to the following observation: the form of $\bphi_{w_k}(s_t, a_t)$ and $w_k(s_{t+1})$ fit in a linear bandits problem with stochastic reward. More specifically, by setting the action $\xb_t = \bphi_{w_k}(s_t,a_t)$, the reward $y_t = w_k(s_{t+1})$, and the noise $\eta_t = w_k(s_{t+1}) - \la \btheta^*, \bphi_{w_k}(s_t,a_t) \ra$, we have
\begin{align*}
    y_t = \la \btheta^*, \xb_t \ra + \eta_t,
\end{align*}
and 
$
    \EE[\eta_t | \cF_t] = 0$,
    $|\eta_t| \le D$,
    $\|\btheta^*\| \le B$,
    $\| \xb_t \| \le D/2
$.

This setting has been thoroughly studied in \citet{abbasi2011improved}. Define the confidence set as 
\begin{align} \label{eqn:hat-C}
	    \hat{\cC}_t = \bigg\{ 
	    \btheta : \Big \| 
	     \hbSigma_{t}^{1/2} (\btheta - \hbtheta_t)
	    \Big \| \le \hbeta_t
	    \bigg\},
\end{align}
where $\hbeta_t$ is
$
    \hbeta_t
     = 
    D 
    \sqrt{d
    \log \big( {(\lambda + t  D^2)}/{(\delta\lambda)} \big)
    } 
    + 
    \sqrt{\lambda} B
$.

In the later section, we show that the true parameter $\btheta^*$ belongs to $\hat{\cC}_t$ with high probability. 
Therefore, $\hat\cC_t$ is a valid confidence set of $\btheta^*$, thus can be fed into the EVI procedure.


In summary, at each time step, $\algname$ always takes action $a_t$ under the policy $\pi_t$,  receives the next state $s_{t+1}$ and refines its confidence set with the new observation (Lines~\ref{line14}-\ref{line23} of Algorithm~\ref{algorithm}). 
When it collects enough new observations for a better estimation than the previous one (Line~\ref{line5}), $\algname$ calls the subroutine EVI with the tighter confidence set and obtains a better policy $\pi_t$ (Lines~\ref{line8}-\ref{line12}).


\subsection{$\algname$ with Bernstein-type bonus}

$\algname$ with \textbf{Bernstein-type bonus (OPTION 2)} is an improved variant of the basic version of $\algname$ with Hoeffding-type bonus.
The key difference is that here we are trying to utilize the variance information of the value functions to construct a tighter confidence set of $\btheta^*$. Recall the construction of $\hat\cC_t$ in Section \ref{sec:hoff}, we set the center of the confidence set $\hat\btheta_t$ as the solution to a ridge regression problem. The motivation is that $\algname$ relies on the fact that the noise $\eta_t$ is $D$-bounded and therefore $D^2$-sub-Gaussian. 
However, the variance of $\eta_t$ is not necessarily that large. In fact, by the law of total variance \citep{azar2013minimax}, we know that on average, the variance of the noise is roughly on the order of $D$ rather than $D^2$. To enable the application of total variance, for $t_k<t\leq t_{k+1}$, we set $\hat\btheta_t$ as the solution to the following \emph{weighted ridge regression problem}:
\begin{align}
 \argmin_{\btheta}{\sum_{i=1}^{k}\sum_{j = t_k}^t} \frac{[\la \bphi_{w_i}(s_j, a_j), \btheta\ra - w_i(s_{j+1})]^2}{\bar\sigma_j^2}  + \lambda \|\btheta\|_2^2,\notag
\end{align}
where $\bar\sigma_j^2$ (Line~\ref{line19}) is an estimation of the \emph{variance} of $w_i(s_{j+1})$. Choosing the weights as the inverse of the variances can guarantee that the estimator has the lowest variance, similar to the best linear unbiased estimator (BLUE) estimator \citep{henderson1975best} for linear regression with fixed design. After obtaining $\hat\btheta_t$, similar to $\algname$ with Hoeffding-type bonus, we construct the following confidence set which contains $\btheta^*$,
\begin{align} \label{eqn:hat-C-B}
	    \hat{\cC}_t = \Big\{ 
	    \btheta : \Big \| 
	     \hbSigma_{t}^{1/2} (\btheta - \hbtheta_t)
	    \Big \| \le \hbeta_t
	    \Big\},
\end{align}
where $\hat\bSigma_t$ is the covariance matrix of contexts $\bphi_{w_i}(s_j, a_j)$ weighted by $\bar\sigma_j^{-2}$, recursively defined in Line \ref{line20} of Algorithm~\ref{algorithm}; $\hbeta_t$ is defined as follows:
\begin{align*}
    \hbeta_t
    & :=
    8 \sqrt{ d \log(1 + t/4 \lambda) \log(4t^2/\delta)}
    \\
    & \qquad +
    4 \sqrt{d} \log(4t^2/ \delta)
    +
    \sqrt{\lambda} B,
\end{align*}
In our later analysis, we require $\bar\sigma_t^2$ to satisfy: (1) it upper bounds the true variance $[\VV w_k](s_t, a_t)$ with high probability; and (2) it is strictly positive. To fulfil these two requirements, we first build a valid estimator $[\bar{\VV}_t w_k ](s_t, a_t)$ for the true variance $[\VV w_k](s_t, a_t)$, based on the following fact:
\begin{align}
   [\VV w_k](s_t, a_t)
   & = [\PP w_k^2] (s_t, a_t) - [\PP w_k(s_t, a_t)]^2
   \notag \\
   & = \la\bphi_{w_k^2}(s_t, a_t), \btheta^*\ra - \la\bphi_{w_k}(s_t, a_t), \btheta^* \ra^2,\label{eq:help}
\end{align}
\eqref{eq:help} suggests that the variance of $w_k$ can be regarded as the combination of two linear functions $\la\bphi_{w_k^2}(s_t, a_t), \btheta^*\ra$ and $\la\bphi_{w_k}(s_t, a_t), \btheta^* \ra$, with respect to different feature mappings. Therefore, we define our variance estimator $[\bar{\VV}_t w_k ](s_t, a_t)$ as follows
\begin{align}
    [ \bar{\VV}_t w_k ] (s_t, a_t) 
     : & = 
    \Big[ 
    \big \la \bphi_{w_t^2}(s_t, a_t), \tbtheta_t  \big \ra
    \Big]_{[0,D^2/4]}
    \notag \\
    & \qquad -
    \Big[  \big \la \bphi_{w_t}(s_t, a_t), \btheta_t  \big \ra
    \Big]^2_{[0,D/2]}, \label{eqn:estimated-variance} 
\end{align}
where $\tilde\btheta_t$ is another estimator for $\btheta^*$.  Specifically, we choose $\tilde\btheta_t$ as the solution to the ridge regression problem with contexts $\{\bphi_{w_{i}^2}(s_{j}, a_{j})\}$ and targets $\{w_{i}^2(s_{j+1})\}$, whose closed-form solution is in Line~\ref{line22} of Algorithm~\ref{algorithm}. Based on this initial variance estimator, we build our final variance estimator $\bar\sigma_t^2$ as in Line \ref{line19}, where the correction term $E_t$ is defined as 
\begin{align}
    E_t
    & :=
     \min 
     \Big \{
     D^2/4,
    \tbeta_t
    \Big\|
    \tbSigma_t^{-1/2}
    \bphi_{w_t^2}(s_t, a_t)
    \Big\|
    \Big \}
    \notag \\
    & \qquad +
    \min 
     \Big \{
     D^2/4,
    D 
    \cbeta_t
    \Big \|
    \hbSigma_t^{-1/2} \bphi_{w_t}(s_t, a_t)
    \Big \|
    \Big \}
    \label{eqn: variance-upper-bound},
\end{align}
where $\tilde\bSigma_t$ is the covariance matrix of features $\bphi_{w_i^2}(s_j, a_j)$ recursively defined in Line \ref{line20}, $\cbeta_t$ and $\tbeta_t$ are defined respectively as follows
\begin{align*}
    \cbeta_t
    & :=
    8 d \sqrt{ \log(1 + t/4 \lambda) \log(4t^2/\delta)}
    \\
    &\qquad +
    4 \sqrt{d} \log(4t^2/ \delta)
    +
    \sqrt{\lambda} B,
    \\
    \tbeta_t
    & :=
    2 D^2 \sqrt{d \log(1 + tD^2/4 d \lambda) \log(4t^2/\delta)}
    \\
    &\qquad +
    D^2 \log(4t^2/ \delta)
    +
    \sqrt{\lambda} B.
\end{align*}
It can be verified that such a $\bar\sigma_t^2$ satisfies both conditions discussed above (i.e., larger than the true variance, strictly positive) simultaneously.


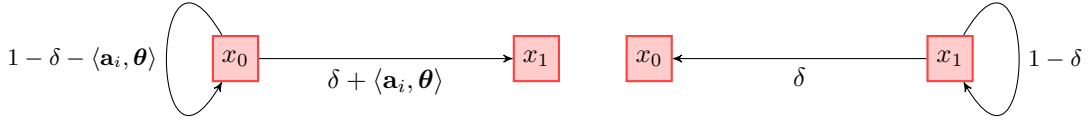
\begin{figure*}[h]
    \centering
    \begin{tikzpicture}[node distance=2cm,>=stealth',bend angle=45,auto]

  \tikzstyle{place}=[rectangle,thick,draw=red!75,fill=red!20,minimum size=6mm]
\tikzset{every loop/.style={min distance=15mm, font=\footnotesize}}
  \begin{scope}[xshift=-3.5cm]
    \node [place] (c1)                                    {$\state_1$};

    \coordinate[left of=c1] (e1) {}; 
    \node [place] (e2) [left of=e1] {$\state_0$};

    \path[->] (e2)
    edge [in=-120,out=120, min distance = 6cm, loop] node[left] {$1- \delta - \la \ab_i, \btheta\ra$} ()
    edge [in=180,out=0] node[below]{$\delta + \la \ab_i, \btheta\ra$} (c1)
    ;

  \end{scope}
  
    \begin{scope}[xshift=2cm]
    \node [place] (c1)                                    {$\state_1$};

    \coordinate[left of=c1] (e1) {}; 
    \node [place] (e2) [left of=e1] {$\state_0$}; 
    \path[->] (c1)
    edge [in=-60,out=60, min distance = 6cm, loop] node[right] {$1- \delta$} ()
    edge [in=0,out=180] node[auto]{$\delta$} (e2)
    ;

  \end{scope}

\end{tikzpicture}
\vspace*{-8mm}
    \caption{Illustration of the hard-to-learn linear mixture MDP considered in Theorem \ref{thm:lower-bound}. The left figure demonstrates the state transition probability starting from $\state_0$ with some specific action $\ab_i$. The right figure demonstrates the state transition probability starting from $\state_1$ with any action. For the detailed construction, see Appendix~\ref{subsec:construction}.}
    \label{fig:hardmdp}
\end{figure*}

\section{MAIN RESULTS}\label{sec:upper-bound}

In this section, we present the regret analysis for Algorithm \ref{algorithm} with both Hoeffding-type exploration bonus (OPTION 1) and Bernstein-type bonus (OPTION 2). 
\begin{theorem} \label{thm:upper-bound}
Setting
$
    \lambda = 1/B^2,
    \epsilon =1 / \sqrt{T}
$,
then for any initial state $s_1$, with probability at least $1 - 2 \delta$, the regret of Algorithm~\ref{algorithm} with Hoeffding-type is bounded as follows:
\begin{align*}
    \text{Regret}(T)
    & = \tilde{O}(d D \sqrt{T}),
\end{align*}
where the $\tilde O(\cdot)$ hides logarithmic terms of $d, D, T$, and $\delta^{-1}$. 
\end{theorem}

Theorem \ref{thm:upper-bound} shows that the regret of Algorithm \ref{algorithm} only depends on the number of rounds $T$, the feature dimension $d$, and the diameter of the communicating MDP $D$. Therefore, Algorithm \ref{algorithm} is statistically efficient for linear mixture MDPs with a finite diameter but large state and action spaces. 

\begin{remark}
The UCRL2 algorithm proposed in \citet{jaksch2010near} has an $\tilde O(|\cS|D\sqrt{|\cA|T})$ regret bound for tabular MDPs with finite state and action spaces and diameter. As a comparison, our \algname enjoys a better upper bound $\tilde{O}(d D \sqrt{T})$ when $d \leq |\cS|\sqrt{|\cA|}$, which suggests that RL with linear function approximation can be more advantageous than vanilla RL algorithms when the underlying MDP has certain nice structures \citep{modi2019sample,yang2019reinforcement}. 
\end{remark}

\begin{theorem} \label{thm:bernstein-upper-bound}
Set
$
    \lambda = 1/B^2,
    \epsilon =1 / \sqrt{T}
$,
then for any initial state $s_1$, with probability at least $1 - 5 \delta$, the regret of Algorithm~\ref{algorithm} with Bernstein bonus is bounded as follows:
\begin{align*}
    \text{Regret}(T)
    & =
    \tilde{O}
    \big(D \sqrt{dT} + d \sqrt{DT} + D d^{7/4} T^{1/4}\big),
\end{align*}
where the $\tilde O(\cdot)$ hides logarithmic terms of $d, D, T$, $\delta^{-1}$. 
\end{theorem}

Suppose $d \geq D$ and $T \geq D^2d^3$, then Theorem \ref{thm:bernstein-upper-bound} suggests that by using the Bernstein-type exploration strategy, the regret bound of Algorithm \ref{algorithm} can be further improved by a factor of $\sqrt{D}$ to be $\tilde O(d\sqrt{DT})$. 

\begin{remark} \label{remark:unknown-D}
It can be shown that the algorithm and analysis can be easily modified to deal with an unknown bounded diameter $D$. More specifically, we can start from a small guess $D'$ and run the algorithm as if it is a valid upper bound. The guess $D'$ will be rejected once EVI returns a 
value function $u_k(s)$ with a span larger than $D'$
. 
Then we retry Line \ref{line9}-\ref{line11} with a doubled guess 
$2D'$. Meanwhile, if no violation happens and the algorithm ends, the proof will go through for $D'$
. 
In the worst case, a guess of at most $2D$ will ensure that EVI obtains a valid estimation, thus introducing a constant factor of up to $2$.
\end{remark}

\noindent\textbf{Comparison with FOPO} 
We would like to do a comparison between $\algname$ in Algorithm \ref{algorithm} and the recently proposed FOPO by \cite{wei2020learning}.
While FOPO is originally proposed for linear MDPs and our algorithm is designed for linear mixture MDPs, since bilinear MDPs \citep{yang2019reinforcement} is a special class of both MDP classes, we can choose the bilinear MDPs as a common ground for comparison.
More specifically, both $\algname$ and FOPO focus on using linear function approximation to learn an infinite-horizon average-reward MDP, and both of them use the \emph{optimism-in-the-face-of-uncertainty} (OFU) principle to learn the optimal value function among a class of plausible MDPs. However, FOPO adapts the Bellman optimality equation assumption and learns the optimal value function by solving a constrained nonconvex optimization problem, which is hard to solve even in the tabular MDP case. In sharp contrast, similar to \cite{jaksch2010near}, $\algname$ adapts the finite-diameter assumption and uses the EVI procedure to find the optimal value function, which is computationally efficient for bilinear MDPs (See \cite{zhou2020provably, zhou2020nearly} for a detailed discussion).
In the setting of bilinear MDPs, we have $\PP(s'|s,a) = \sum_{j=1}^{d} \theta_j \psi_j(s')  \mu_j(s,a)$. We can rewrite the transition probability as a linear mixture MDP ($\odot$ denotes the Hadamard product):
\begin{align*}
    \PP(s'|s,a)
    & =
    \big \la
    \btheta
    ,
    \bpsi(s) \odot \bmu(s,a)
    \big \ra,
\end{align*}
and as a linear MDP:
\begin{align*}
    \PP(s'|s,a)
    & =
    \big \la
    \btheta \odot \bpsi(s)
    ,
    \bmu(s,a)
    \big \ra.
\end{align*}
Therefore, the regret bound of $\algname$ and FOPO are $O(d \sqrt{DT})$ and $O(d^{3/2}  \sqrt{\text{span}(h^*) T})$, respectively. The diameter $D$ and the span $\text{span}(h^*)$ are closely related and are often of the same scale. This immediately shows that $\algname$ achieves a smaller regret and is statistically more efficient than FOPO for bilinear MDPs.

The following theorem presents a matching lower bound of regret for infinite-horizon average-reward linear mixture MDPs.

\begin{theorem} \label{thm:lower-bound}
Suppose $d \ge 2$, $T \ge 16 d^2 D / 2025$ and $B >1$. Then for any algorithm $A$, there exists a $B$-bounded MDP $M_{\btheta}$ as illustrated in Figure \ref{fig:hardmdp} such that 
\begin{align*}
    \EE[\text{Regret}(M_{\btheta}, A, s, T)]
    \ge 
    d \sqrt{DT}/2025.
\end{align*}
\end{theorem}
If we set $\delta = T^{-1}$ in Theorem~\ref{thm:bernstein-upper-bound}, we can see the expected regret upper bound is of the order $\tilde{O}(d\sqrt{DT})$, which differs from the lower bound only by logarithmic factors. The dependence on $d,D,T$ matches with each other and thus implies the upper bound cannot be improved.

\begin{remark}
An interesting observation is,  \citet{jaksch2010near} proved in the tabular setting that for any algorithm, the regret is lower bounded by $\Omega(\sqrt{|\cS||\cA|DT})$. Since tabular MDPs can be regarded as a special case of linear mixture MDP with a $d = |\cS|^2|\cA|$ dimensional feature mapping, \citet{jaksch2010near}'s construction actually yields a slightly worse lower bound $\Omega(\sqrt{|\cS|^2|\cA|DT}) = \Omega(\sqrt{dDT})$, for the general linear mixture MDPs. And while our construction gains inspiration from \citet{jaksch2010near}, $\Omega(d\sqrt{DT})$ is tighter than the induced lower bound by a factor of $d^{1/2}$. This also indicates that our MDP construction is nontrivial. 
\end{remark}

\begin{remark}
Our lower bound can also imply a lower bound for the \emph{linear MDP} setting studied by \citet{wei2020learning}. By a similar construction of the hard-to-learn MDP instance, we can prove an $\Omega(d\sqrt{\text{sp}(h^*)T})$ lower bound for learning linear MDPs. The detailed reasoning is deferred to the appendix. This suggests that there still exists a gap to be removed under the linear MDP setting considered by \citet{wei2020learning}.
\end{remark}

\section{CONCLUSION AND FUTURE WORK} \label{sec:conclusion}

In this paper, we push the frontier of learning infinite-horizon average-reward Markov Decision Process with linear function approximation.
We propose the first algorithm that achieves nearly minimax optimal regret. 
Our lower bound can also imply a lower bound for linear MDPs, which is of independent interest. Our current algorithms and results are limited to MDPs with finite diameter. In the future work, it is possible to relax this constraint, and extend our algorithms to deal with a milder assumption called finite span assumption \citep{bartlett2009regal, zhang2019regret, wei2020learning}, while still achieving the minimax optimality.

\subsubsection*{Acknowledgements}
We thank the anonymous reviewers for their helpful comments. 
Part of this work was done when DZ and QG participated the Theory of Reinforcement Learning program at the Simons Institute for the Theory of Computing in Fall 2020. YW, DZ and QG are partially supported by the National Science Foundation CIF-1911168, CAREER Award 1906169, IIS-1904183 and AWS Machine Learning Research Award. The views and conclusions contained in this paper are those of the authors and should not be interpreted as representing any funding agencies.




\bibliography{ref}

\newpage
\onecolumn

\appendix

\section{Proof of Theorem \ref{thm:upper-bound}} \label{sec:proof-upper-bound}

We use $K(T)-1$ to denote the value of the counter $k$ when Algorithm~\ref{algorithm} finishes, and $t_{K(T)} = T + 1$ for convenience. By these notation, the learning process from $t=1$ to $t=T$ can be divided into $K(T)$ episodes.

The following lemma, proved by~\citet{jaksch2010near}, states that EVI (Algorithm~\ref{alg:EVI}) always outputs a near-optimal policy and an optimistic model.
\begin{lemma}[Theorem 7 and Equation (12) in \citealt{jaksch2010near}] \label{lemma:EVI}
Stopping the extended value iteration when
\begin{align*}
    \max_{s \in \cS}
    \big \{ 
    u^{(i+1)}(s) - u^{(i)}(s)
    \big \} 
    -
    \min_{s \in \cS}
    \big \{ 
    u^{(i+1)}(s) - u^{(i)}(s)
    \big \} 
    <
    \epsilon,
\end{align*}
the greedy policy $\tilde{\pi}$ with respect to $u^{(i)}$ is $\epsilon$-optimal, namely
\begin{align} \label{eqn:EVI-optimal}
    \tilde{\rho} 
    &:=
    \rho(\tilde{M}, \tilde{\pi})
    \ge 
    \max_{\pi, M \in \cM}
    \rho(M, \pi)
    - \epsilon.
\end{align}
Here, $\tilde{M}$ means the Markov Decision Process (MDP) determined by the parameterized transition probability, e.g. $\PP_k(\cdot| s,a) = \la \bphi(\cdot | s,a), \btheta_k(s,a) \ra$.
For each $M \in \cM$, $M$ is an MDP with parameter from the confidence set. $\cM$ is assumed to contain the true transition model.

Moreover, we have $\forall s \in \cS$,
\begin{align} \label{eqn:EVI-gap}
    |
    u^{(i+1)}(s) - u^{(i)}(s) - \tilde{\rho}
    |
    \le \epsilon.
\end{align}
\end{lemma}

The next lemma describes that indeed, the confidence sets we constructed contain the true parameter with high probability.
\begin{lemma}\label{lemma:theta-ball}
With probability at least $1-\delta$, for all $0 \leq k\leq K(T)-1$, we have $\btheta^* \in \hat{\cC}_{t_k}$.
\end{lemma}
\begin{proof}
See Section~\ref{sec:proof-lemma-theta-ball}.
\end{proof}

The number of episodes in our algorithm turns out can be bounded as follows:
\begin{lemma}\label{lemma:boundk}
We have $K(T) \leq d\log[(2\lambda +2 TD^2)/\lambda]$. 
\end{lemma}
\begin{proof}
See Section~\ref{sec:proof-lemma-boundk-alg1}.
\end{proof}

The rest lemmas either is standard concentration inequalities or is from the works regarding linear bandit problems.
\begin{lemma}[Azuma–Hoeffding inequality]\label{lemma:azuma}
Let $\{X_k\}_{k=0}^{\infty}$ be a
discrete-parameter real-valued martingale sequence such that for every $k\in \NN$, the condition $|X_k-X_{k-1}|\leq \mu$ holds for some non-negative constant $\mu$. Then with probability at least $1-\delta$, we have 
\begin{align}
    X_n-X_0 \leq \mu\sqrt{2n \log 1/\delta}.\notag
\end{align} 
\end{lemma}

\begin{lemma}[Lemma 11 in \citealt{abbasi2011improved}]\label{lemma:sumcontext}
For any $\{\xb_t\}_{t=1}^T \subset \RR^d$ satisfying that $\|\xb_t\|_2 \leq L$, let $\Ab_0 = \lambda \Ib$ and $\Ab_t = \Ab_0 + \sum_{i=1}^{t-1}\xb_i\xb_i^\top$, then we have
\begin{align}
    \sum_{t=1}^T \min\{1, \|\xb_t\|_{\Ab_{t-1}^{-1}}\}^2 \leq 2d\log\frac{d\lambda+TL^2}{d \lambda}.\notag
\end{align}
\end{lemma}

\begin{lemma}[Lemma 12 in \citealt{abbasi2011improved}]\label{lemma:det}
Suppose $\Ab, \Bb\in \RR^{d \times d}$ are two positive definite matrices satisfying that $\Ab \succeq \Bb$, then for any $\xb \in \RR^d$, $\|\xb\|_{\Ab} \leq \|\xb\|_{\Bb}\cdot \sqrt{\det(\Ab)/\det(\Bb)}$.
\end{lemma}

\begin{proof}[Proof of Theorem \ref{thm:upper-bound}]
We first split the regret into each episode. Denote the regret in episode $k$ as $\Delta_k$, and we have
\begin{align*}
    \Delta_k 
    & := 
    \sum_{t = t_k}^{t_{k+1} - 1} 
    [
    \rho^* - r(s_t, a_t)
    ]
    \\ 
    & \le 
    (t_{k+1} - t_k) \epsilon 
    +
    \sum_{t = t_k}^{t_{k+1} - 1} 
    [
    {\rho}_k - r(s_t, a_t)
    ]
    \\ 
    & \le 
    2 (t_{k+1} - t_k) \epsilon 
    + 
    \sum_{t = t_k}^{t_{k+1} - 1}  
    \big [ 
    \sum_{s' \in \cS}
    {\PP}_k (s' | s_t, a_t) u_k(s')  - u_k(s_t)
    \big ]
    \\
    & =
    2 (t_{k+1} - t_k) \epsilon 
    + 
    \sum_{t = t_k}^{t_{k+1} - 1} 
    \big [ 
    \sum_{s' \in \cS}
    {\PP}_k (s' | s_t, a_t) w_k(s')  - w_k(s_t)
    \big ]
    \\
    & =
    2 (t_{k+1} - t_k) \epsilon 
    + 
    \sum_{t = t_k}^{t_{k+1} - 1} 
    \big [ 
    [{\PP}_k  w_k](s_t, a_t)  - w_k(s_t)
    \big ].
\end{align*}
The first inequality is due to the $\epsilon-$optimality of the EVI algorithm (Lemma~\ref{lemma:EVI}). 
The second inequality is due to  \eqref{eqn:EVI-gap} and substitute the iteration rule $u^{(i+1)}(s) \leftarrow \max_{a \in \cA} 
        \bigg\{ 
        r(s,a)
        + 
        \max_{\btheta \in \cC \cap \cB}
        \Big\{ 
          \big \la 
          \btheta,  \bphi_{u^{(i)}}(s,a)
          \big \ra
        \Big\} 
        \bigg\}$.
Here, notice that we denote $\PP_k (s' |s_t, a_t) = \la \btheta_k(s_t,a_t),  \bphi(s'|s_t,a_t) \ra$ and $\btheta_k(s_t,a_t) = \argmax_{\btheta \in \cC \cap \cB}
        \Big\{ 
          \big \la 
          \btheta,  \bphi_{u^{(i)}}(s,a)
          \big \ra
        \Big\} $. By the definition of $\pi_k$, $a_t$ achieves the outer maximum in the iteration rule of $u^{(i+1)}$.
The second last equality is due to the fact that adding a bias to $u_k$ won't change the difference, as what has been done in Algorithm~\ref{algorithm}. So we subtract $(\max_s u_k(s) + \min_s u_k(s) )/2$ from $u_k(s)$.
The last equality is a shorthand. Notice that since the span of $u_k(s)$ is $D$, we have $|w_k(s)| \le D/2$.

Summing over all episodes, we further have
\begin{align*}
    \sum_{k=0}^{K(T)-1}
    \Delta_k
    & = 
    2T \epsilon +
    \underbrace{
    \sum_{k=0}^{K(T)-1}
    \sum_{t = t_k}^{t_{k+1} - 1} 
    \big [ 
    [{\PP}_k w_k](s_t,a_t) - 
    [{\PP} w_k](s_t,a_t)
    \big ]
    }_{I_1}
    \\
    & \qquad 
    +
    \underbrace{
    \sum_{k=0}^{K(T)-1}
    \sum_{t = t_k}^{t_{k+1} - 1} 
    \big [ 
    [{\PP} w_k](s_t,a_t) - w_k(s_t)
    \big ]
    }_{I_2}.
\end{align*}
The first term can be controlled following the idea of bounding the regret of linear bandit. We have that with probability $1-\delta$,
\begin{align*}
    I_1
    & = 
    \sum_{k=0}^{K(T)-1}
    \sum_{t = t_k}^{t_{k+1} - 1} 
    \big \la 
    {\btheta}_k - \btheta^*
    ,
    \bphi_{w_k}(s_t, a_t)
    \big \ra
    \\
    & \le 
    \sum_{k=0}^{K(T)-1}
    \sum_{t = t_k}^{t_{k+1} - 1} 
    \big(
    \|
    {\btheta}_k - {\hbtheta}_k
    \|_{\hbSigma_{t}}
    +
    \|
    \btheta^* - {\hbtheta}_k
    \|_{\hbSigma_{t}}
    \big)
    \big \|
    \bphi_{w_k}(s_t, a_t)
    \big \|_{\hbSigma_{t}^{-1}}
    \\
    & \le 
    \sum_{k=0}^{K(T)-1}
    \sum_{t = t_k}^{t_{k+1} - 1} 
    2\big(
    \|
    {\btheta}_k - \hat{\btheta}_k
    \|_{\hbSigma_{t_k}}
    +
    \|
    \btheta^* - \hat{\btheta}_k
    \|_{\hbSigma_{t_k}}
    \big)
    \big \|
    \bphi_{w_k}(s_t, a_t)
    \big \|_{\hbSigma_{t}^{-1}}
    \\
    & \le 
    \sum_{k=0}^{K(T)-1}
    \sum_{t = t_k}^{t_{k+1} - 1} 
    4 \hbeta_{T}
    \big \|
    \bphi_{w_k}(s_t, a_t)
    \big \|_{\hbSigma_{t}^{-1}}
    .
\end{align*}
The first inequality is due to first applying Cauchy-Schwartz inequality and then the triangle inequality. The second inequality is due to Lemma~\ref{lemma:det} and the fact that for $t_k \le t < t_{k+1}$ $\det(\bSigma_{t}) \le \det(\bSigma_{t_{k+1}}) \le 2 \det(\bSigma_{t_{k}})$. The third inequality is due to Lemma~\ref{lemma:theta-ball} and the fact that $\{\hbeta_t \}_t$ is increasing.

Meanwhile, for each term in $I_1$, we also have that due to the fact $|w_k(s)| \le D/2$,
\begin{align*} 
    [{\PP}_k w_k] (s_t,a_t) - [{\PP} w_k] (s_t,a_t)
    \le 
    D.
\end{align*}
Therefore, we have
\begin{align*}
    I_1
    & \le 
    \sum_{k=0}^{K(T)-1}
    \sum_{t = t_k}^{t_{k+1} - 1} 
    \min 
    \Big \{
    D
    ,
    4 \hbeta_{T}
    \big \|
    \bphi_{w_k}(s_t, a_t)
    \big \|_{\hbSigma_{t}^{-1}}
    \Big \}
    \\
    & \le 
     4 \hbeta_{T}
    \sum_{k=0}^{K(T)-1}
    \sum_{t = t_k}^{t_{k+1} - 1} 
    \min 
    \Big \{
    1
    ,
    \big \|
    \bphi_{w_k}(s_t, a_t)
    \big \|_{\hbSigma_{t}^{-1}}
    \Big \}
    \\
    & \le 
    4 \hbeta_{T}
    \sqrt{ 
    T
    \sum_{k=0}^{K(T)-1}
    \sum_{t = t_k}^{t_{k+1} - 1} 
    \min 
    \Big \{
    1
    ,
    \big \|
    \bphi_{w_k}(s_t, a_t)
    \big \|_{\hbSigma_{t}^{-1}}^2
    \Big \}
    }
    \\
    & \le 
    4 \hbeta_{T}
    \sqrt{ 
    T
    \cdot 
    2 d
    \log 
    \bigg(
    \frac{d\lambda + T D^2}{d\lambda}
    \bigg)
    }
    \\
    & \le 
    6 \hbeta_{T}
    \sqrt{ 
    d T
    \log 
    \bigg(
    \frac{d\lambda + T D^2}{d\lambda}
    \bigg)
    }.
\end{align*}
The second inequality is due to the fact $D \le 4 \hbeta_{T}$. The third is due to Cauchy-Schwartz inequality. The fourth is due to Lemma~\ref{lemma:sumcontext}.

The second term, can be controlled by the concentration of martingale. With probability $1-\delta$, 
\begin{align*}
    I_2
    & = 
    \sum_{k=0}^{K(T)-1}
    \sum_{t = t_k}^{t_{k+1} - 1} 
    \big [ 
    [{\PP} w_k](s_t,a_t) - w_k(s_t)
    \big ]
    \\ 
    & = 
    \sum_{k=0}^{K(T)-1}
    \Bigg [ 
    \sum_{t = t_k}^{t_{k+1} - 1} 
    \big(
    [{\PP} w_k](s_t,a_t) - w_k(s_{t+1})
    \big) 
    - w_k(s_{t_k})
    + w_k(s_{t_{k+1}})
    \Bigg ]
    \\
    & \le 
    \sum_{k=0}^{K(T)-1}
    \Bigg [ 
    \sum_{t = t_k}^{t_{k+1} - 1} 
    \big(
    [{\PP} w_k](s_t,a_t) - w_k(s_{t+1})
    \big) 
    \Bigg]
    +
    D \cdot K(T)
    \\
    &
    \le 
    D \sqrt{2T \log (1/\delta)}
    +
    D \cdot K(T),
\end{align*}
where the first inequality holds because $|w_k(s)| \le D/2$; the second inequality is due to Lemma~\ref{lemma:azuma}.

Therefore, the total regret is bounded by
\begin{align*}
    \text{Regret}(T)
    & =
    \sum_{k=0}^{K(T) - 1}
    \Delta_k 
     \le 
    2 T \epsilon 
    +
    6 \hbeta_T
    \sqrt{ 
    d T
    \log 
    \bigg(
    \frac{\lambda + T D^2}{\lambda}
    \bigg)
    }
    +
    D \sqrt{2T \log (1/\delta)}
    +
    D \cdot K(T).
\end{align*}
If we set
\begin{align*}
    \hbeta_t
    & = 
    D 
    \sqrt{ d
    \log \bigg( \frac{\lambda + t  D^2}{\delta\lambda} \bigg)
    } 
    + 
    \sqrt{\lambda} B,
\end{align*}
and 
\begin{align*}
    \epsilon = \frac{1}{\sqrt{T}},
\end{align*}
then by taking union bound we have with probability $1- 2 \delta$,
\begin{align*}
    \text{Regret}(T)
    & \le 
    2\sqrt{T}
    +
    Dd \sqrt{T} \cdot \tilde{O}(1)
    +
    B \sqrt{\lambda dT}  \cdot \tilde{O}(1)
    +
    D \sqrt{2T \log (1/\delta)}
    +
    D d \log 
    \bigg( 
    \frac{2\lambda + 2dTD^2}{\lambda}
    \bigg) 
    \\ 
    & \le 
    \tilde{O}(Dd \sqrt{T} ),
\end{align*}
where $\tilde{O}(1)$ hides the log factor, the last inequality holds since we set $\lambda = 1/B^2$.

\end{proof}

\section{Proof of Theorem \ref{thm:bernstein-upper-bound}} \label{sec:proof-bernstein-upper-bound}

Most part of the proof resembles that of Theorem~\ref{thm:upper-bound}. The additional part is essentially about the new concentration results from variance-aware linear bandit problem. As previously defined, we use $K(T)-1$ to denote the value of the counter $k$ when Algorithm~\ref{algorithm} finishes, and $t_{K(T)} = T + 1$ for convenience. By these notations, the learning process from $t=1$ to $t=T$ can be divided into $K(T)$ episodes.

The first lemma provides a better confidence set given the information of the noise's variance.
\begin{lemma}[Bernstein inequality for vector-valued martingales \citep{zhou2020nearly}]\label{lemma:concentration_variance}
Let $\{\cG_{t}\}_{t=1}^\infty$ be a filtration, $\{\xb_t,\eta_t\}_{t\ge 1}$ a stochastic process so that
$\xb_t \in \RR^d$ is $\cG_t$-measurable and $\eta_t \in \RR$ is $\cG_{t+1}$-measurable. 
Fix $R,L,\sigma,\lambda>0$, $\bmu^*\in \RR^d$. 
For $t\ge 1$ 
let $y_t = \la \bmu^*, \xb_t\ra + \eta_t$ and
suppose that $\eta_t, \xb_t$ also satisfy 
\begin{align}
    |\eta_t| \leq R,\ \EE[\eta_t|\cG_t] = 0,\ \EE [\eta_t^2|\cG_t] \leq \sigma^2,\ \|\xb_t\|_2 \leq L.\notag
\end{align}
Then, for any $0 <\delta<1$, with probability at least $1-\delta$ we have 
\begin{align}
    \forall t>0,\ \bigg\|\sum_{i=1}^t \xb_i \eta_i\bigg\|_{\Zb_t^{-1}} \leq \beta_t,\ \|\bmu_t - \bmu^*\|_{\Zb_t} \leq \beta_t + \sqrt{\lambda}\|\bmu^*\|_2,\label{eq:concentration_variance:xx}
\end{align}
where for $t\ge 1$,
 $\bmu_t = \Zb_t^{-1}\bbb_t$,
 $\Zb_t = \lambda\Ib + \sum_{i=1}^t \xb_i\xb_i^\top$,
$\bbb_t = \sum_{i=1}^ty_i\xb_i$
 and
 \[
\beta_t = 8\sigma\sqrt{d\log(1+tL^2/(d\lambda)) \log(4t^2/\delta)} + 4R \log(4t^2/\delta)\,.
\]
\end{lemma}

The number of episodes is bounded almost in the same way as in Lemma~\ref{lemma:boundk}:
\begin{lemma}\label{lemma:bernstein-boundk}
Let $K(T)$ be as defined above. Then, $K(T) \leq 2d\log (1+Td/\lambda)$. 
\end{lemma}
\begin{proof}
See Section~\ref{sec:proof-bernstein-boundk}.
\end{proof}

The variance term is defined as
\begin{align*}
    [\VV w_k] (s_t, a_t)
    & := 
    \EE_{s' \sim \PP(\cdot| s_t,a_t)}
    [w_k^{2}(s')]
    -
    \EE_{s' \sim \PP(\cdot| s_t,a_t)}[w_k(s')]^{2}.
\end{align*}
The following lemma states that with high probability the estimated variance is close the the true variance.
\begin{lemma} \label{lemma:confidence-set}
With probability $1-3\delta$, we have for all $ 1 \le t \le T$, 
\begin{align*}
    \btheta^* \in \hat{\cC}_t \cap \cB,
    \big | [\bar{\VV}_t w_k](s_t,a_t)
    -
    [\VV w_k](s_t,a_t) \big |
    \le 
    E_t.
\end{align*}
We denote the event above by $\cE_0$, and $\PP(\cE_0) \ge 1 - 3 \delta$.
\end{lemma}
\begin{proof}
See Section~\ref{sec:proof-lemma-confidence-set}.
\end{proof}

Now, we define other events:
\begin{align*}
    \cE_1 
    & :=
    \Big\{
    \sum_{k=0}^{K(T)-1}
    \sum_{t=t_k}^{t_{k+1}-1}
    \big[
    \EE_{s' \sim \PP(\cdot | s_t, a_t)}[w_k(s')^2] 
    - w_k^2(s_{t+1})
    \big]
    \le 
    (D^2/4) \sqrt{2T \log(1/\delta)}
    \Big\}
    \\
     \cE_2 
    & :=
    \Big\{
    \sum_{k=0}^{K(T)-1}
    \sum_{t=t_k}^{t_{k+1}-1}
    \big[
    \EE_{s' \sim \PP(\cdot | s_t, a_t)}[w_k(s')] 
    - w_k(s_{t+1})
    \big]
    \le 
    (D/2) \sqrt{2T \log(1/\delta)}
    \Big\}
\end{align*}
By the Azuma-Hoeffding inequality(Lemma~\ref{lemma:azuma}), we have $\PP(\cE_1) \ge 1 - \delta$ and $\PP(\cE_2) \ge 1 - \delta$.

The next lemma characterizes the total variance.
\begin{lemma}\label{lemma:total-variance}
Under the events $\cE_0$ and $\cE_1$, we have
\begin{align*}
    \sum_{k=0}^{K(T)-1}
    \sum_{t=t_k}^{t_{k+1}-1}
    [\VV w_k](s_t,a_t) 
    &
    \le 
    (D^2/4) \sqrt{2T \log(1/\delta)}
    +
    (K(T) + 1) (D^2/4)
    +
    2DT
    +
    D^2 \hbeta_T 
    \sqrt{
    T
    2d 
    \log(1 + T/ \lambda)
    }
    .
\end{align*}
\end{lemma}
\begin{proof}
See Section~\ref{sec:proof-lemma-total-variance}.
\end{proof}

The following lemma serves as a wrapper of calculating the total estimation error. 
\begin{lemma} \label{lemma:boundE}
Under the event $\cE_0$, we have
\begin{align*}
    \sum_{t=1}^{T} E_t
    & \le 
    \tbeta_T
    \sqrt{
    2Td
    \log(1 + TD^2/4d\lambda)
    }
    +
    D^2
    \cbeta_T
    \sqrt{
    2Td
    \log(1 + T/\lambda)
    }.
\end{align*}
\end{lemma}
\begin{proof}
See Section \ref{sec:proof-lemma-boundE}.
\end{proof}

Now we are ready to show the regret upper bound.
\begin{proof}
We first follow the same procedure as \citet{jaksch2010near} did to decompose the regret and tackle each term respectively.

We have
\begin{align*}
    \regret(T) 
    & := 
    \sum_{k=0}^{K(T)-1}
    \sum_{t=t_k}^{t_{k+1}-1}
    [
    \rho^* - r(s_t, a_t)
    ]
    \\ 
    & \le 
    T \epsilon 
    +
    \sum_{k=0}^{K(T)-1}
    \sum_{t=t_k}^{t_{k+1}-1}
    [
    {\rho}_k - r(s_t, a_t)
    ]
    \\ 
    & \le 
    2 T \epsilon 
    + 
    \sum_{k=0}^{K(T)-1}
    \sum_{t=t_k}^{t_{k+1}-1}
    \big [ 
    \EE_{s' \sim {\PP}_k (\cdot | s_t, a_t)}
     [u_k(s')]  - u_k(s_t)
    \big ]
    \\
    & =
    2 T \epsilon 
    + 
    \sum_{k=0}^{K(T)-1}
    \sum_{t=t_k}^{t_{k+1}-1} 
    \big [ 
    \EE_{s' \sim {\PP}_k (\cdot | s_t, a_t)}
     [w_k(s')]  - w_k(s_t)
    \big ]
    \\
    & =
    2 T \epsilon 
    + 
    \sum_{k=0}^{K(T)-1}
    \sum_{t=t_k}^{t_{k+1}-1}
    \big [ 
    [\PP_k w_k](s_t, a_t)  - w_k(s_t)
    \big ].
\end{align*}
The first inequality is due to the $\epsilon$-optimality of the EVI algorithm. 
The second inequality is due to  (12) in \citet{jaksch2010near}.
The third inequality is due to the fact that add a bias to $u_t$ won't change the difference, as done in Algorithm~\ref{algorithm}. So we subtract $(\max_s u_t(s) + \min_s u_t(s) )/2$ from $u_t(s)$.
The last equality is a shorthand. It can be further decomposed into:
\begin{align*}
    \sum_{k=0}^{K(T)-1}
    \sum_{t=t_k}^{t_{k+1}-1}
    \big [ 
    [\PP_k w_k](s_t, a_t)  - w_k(s_t)
    \big ]
    & = 
    \underbrace{
    \sum_{k=0}^{K(T)-1}
    \sum_{t=t_k}^{t_{k+1}-1}
    \big [ 
    [\PP_k w_k](s_t, a_t)  - [\PP w_k](s_t, a_t)
    \big ]
    }_{I_1}
    \\
    & \qquad 
    +
    \underbrace{
    \sum_{k=0}^{K(T)-1}
    \sum_{t=t_k}^{t_{k+1}-1}
    \big [ 
    [\PP w_k](s_t, a_t)  - w_k(s_t)
    \big ]
    }_{I_2}.
\end{align*}
We deal with the second term $I_2$ first:

The second term, can be controlled by the concentration of martingale. In fact, $\cE_2$ defined above exactly characterizes the concentration. Under event $\cE_2$, we have
\begin{align*}
    I_2
    & = 
    \sum_{k=0}^{K(T)-1}
    \sum_{t=t_k}^{t_{k+1}-1}
    \big [ 
    [\PP w_k](s_t, a_t)  - w_k(s_t)
    \big ]
    \\
    & = 
    \sum_{k=0}^{K(T)-1}
    \Bigg[
    \sum_{t=t_k}^{t_{k+1}-1}
    \big [ 
    [\PP w_k](s_t, a_t)  - w_k(s_{t+1})
    \big ]
    -w_k(s_{t_k})
    +w_k(s_{t_{k+1}})
    \Bigg]
    \\
    & \le  
    \sum_{k=0}^{K(T)-1}
    \sum_{t=t_k}^{t_{k+1}-1}
    \big [ 
    [\PP w_k](s_t, a_t)  - w_k(s_{t+1})
    \big ]
    +
    K(T) \cdot D
    \\
    &
    \le 
    D \sqrt{2T \log (1/\delta)}
    +
    K(T) \cdot D
    \\
    & = 
    \tilde{O} (D \sqrt{T})
    +
    \tilde{O}(Dd),
\end{align*}
where the first inequality holds since $|w_k(\cdot)| \leq D/2$, the second one holds due to the definition of $\cE_2$. 
For term $I_1$,
\begin{align*}
    I_1
    & = 
    \sum_{k=0}^{K(T)-1}
    \sum_{t = t_k}^{t_{k+1} - 1} 
    \big [ 
    [\PP_k w_k](s_t,a_t)
    - 
    [\PP w_k](s_t,a_t)
    \big ]
    \\
    & = 
    \sum_{k=0}^{K(T)-1}
    \sum_{t = t_k}^{t_{k+1} - 1} 
    \big \la 
    {\btheta}_k - \btheta^*
    ,
    \bphi_{w_k}(s_t, a_t)
    \big \ra
    \\
    & \le 
    \sum_{k=0}^{K(T)-1}
    \sum_{t = t_k}^{t_{k+1} - 1} 
    \big(
    \|
    {\btheta}_k - \hat{\btheta}_k
    \|_{\hbSigma_{t}}
    +
    \|
    \btheta^* - \hat{\btheta}_k
    \|_{\hbSigma_{t}}
    \big)
    \big \|
    \bphi_{w_k}(s_t, a_t)
    \big \|_{\hbSigma_{t}^{-1}}
    \\
    & \le 
    2\sum_{k=0}^{K(T)-1}
    \sum_{t = t_k}^{t_{k+1} - 1} 
    \big(
    \|
    {\btheta}_k - \hat{\btheta}_k
    \|_{\hbSigma_{t_k}}
    +
    \|
    \btheta^* - \hat{\btheta}_k
    \|_{\hbSigma_{t_k}}
    \big)
    \big \|
    \bphi_{w_k}(s_t, a_t)
    \big \|_{\hbSigma_{t}^{-1}}
    \\
    & \le 
    4\sum_{k=0}^{K(T)-1}
    \sum_{t = t_k}^{t_{k+1} - 1} 
    \hbeta_{t_k} 
    \big \|
    \bphi_{w_k}(s_t, a_t)
    \big \|_{\hbSigma_{t}^{-1}}
    \\
    & \le 
    4
    \sum_{k=0}^{K(T)-1}
    \sum_{t = t_k}^{t_{k+1} - 1} 
    \hbeta_{t} \barsigma_t
    \big \|
    \bphi_{w_k}(s_t, a_t) / \barsigma_t
    \big \|_{\hbSigma_{t}^{-1}}
    .
\end{align*}
The first inequality is due to first applying Cauchy-Schwartz inequality and then the triangle inequality. 
The second is due to $\det(\hbSigma_t) \le 2 \det(\hbSigma_{t_k})$ and Lemma~\ref{lemma:det}.
The third is due to event $\cE_0$.
The last is due to the fact that $\{ \hbeta_t \}_{t\ge 0}$ is increasing.

\noindent Meanwhile, for each term in $I_1$, we also have that due to $|w_k(s)| \le D/2$,
\begin{align*} 
    [\PP_k w_k](s_t,a_t)
    - 
    [\PP w_k](s_t,a_t)
    \le 
    D.
\end{align*}
Therefore, we have
\begin{align*}
    I_1
    & \le 
    \sum_{k=0}^{K(T)-1}
    \sum_{t = t_k}^{t_{k+1} - 1} 
    \min 
    \Big \{
    D
    ,
    4
    \hbeta_{t} \barsigma_t
    \big \|
    \bphi_{w_k}(s_t, a_t) / \barsigma_t
    \big \|_{\hbSigma_{t}^{-1}}
    \Big \}
    \\
    & \le 
    \sum_{k=0}^{K(T)-1}
    \sum_{t = t_k}^{t_{k+1} - 1} 
    4
    \hbeta_{t} \barsigma_t
    \min 
    \Big \{
    1
    ,
    \big \|
    \bphi_{w_k}(s_t, a_t) / \barsigma_t
    \big \|_{\hbSigma_{t}^{-1}}
    \Big \}
    \\
    & \le 
    4
    \hbeta_{T}
    \sqrt{
    \underbrace{
    \sum_{t = 1}^{T}  
    (
     \barsigma_t)^2
    }_{J_1}}
    \sqrt{
    \underbrace{
    \sum_{k=0}^{K(T)-1}
    \sum_{t = t_k}^{t_{k+1} - 1} 
    \Big \{
    1
    ,
    \big \|
    \bphi_{w_k}(s_t, a_t) / \barsigma_t
    \big \|_{\hbSigma_{t}^{-1}}
    \Big \}
    }_{J_2}
    }.
\end{align*}
The second inequality is due to the fact $D \le 4 \hbeta_t \barsigma_t $. The third is due to Cauchy-Schwartz inequality. 

Note that by Lemma~\ref{lemma:sumcontext}, it is clear that
\begin{align*}
    J_2 
    & \le 
    2d \log(1 + T/\lambda).
\end{align*}
For term $J_1$, 
\begin{align*}
    J_1
    & =
    \sum_{k=0}^{K(T)-1}
    \sum_{t = t_k}^{t_{k+1} - 1} 
    \max \{ D^2/d, [\bar{\VV}_t w_k](s_t,a_t) + E_t \}
    \\
    & \le 
    TD^2/d +
    \sum_{k=0}^{K(T)-1}
    \sum_{t = t_k}^{t_{k+1} - 1} 
    [\bar{\VV}_t w_k](s_t,a_t) 
    +
    \sum_{t=1}^{T} E_t
    \\
    & \le 
    TD^2/d +
    \sum_{k=0}^{K(T)-1}
    \sum_{t = t_k}^{t_{k+1} - 1} 
    [\VV w_k](s_t,a_t) 
    +
    2 \sum_{t=1}^{T} E_t
    \\
    & \le 
    TD^2/d
    +
    (D^2/4) \sqrt{2T \log(1/\delta)}
    +
    (K(T) + 1) (D^2/4)
    +
    2DT
    +
    D^2 \hbeta_T 
    \sqrt{
    T
    2d 
    \log(1 + T/ \lambda)
    }
    \\ 
    & \qquad
    +
    \tbeta_T
    \sqrt{
    2Td
    \log(1 + TD^2/4d\lambda)
    }
    +
    D^2
    \cbeta_T
    \sqrt{
    2Td
    \log(1 + T/\lambda)
    }.
\end{align*}
The second inequality uses Lemma~\ref{lemma:concentration_variance}. The third uses Lemma~\ref{lemma:total-variance} and Lemma~\ref{lemma:boundE}.

Now, based on Lemma~\ref{lemma:bernstein-boundk} we have $K(T) = \tilde{O}(d)$. By definition, we have 
\begin{align*}
    \hbeta_T & = \tilde{O}(\sqrt{d}) \\
    \cbeta_T & = \tilde{O}(d) \\
    \tbeta_T & = \tilde{O}(D^2 \sqrt{d}),
\end{align*}
if we set $\lambda = B^{-2}$.

This means we can express $I_1$ in Big-O notation term by term as:
\begin{align*}
    J_1
    & = 
    \tilde{O}(TD^2/d)
    +
    \tilde{O}(D^2 \sqrt{T})
    +
    \tilde{O}(D^2 d)
    +
    \tilde{O}(DT)
    +
    \tilde{O}(D^2d\sqrt{T})
    +
    \tilde{O}(D^2d\sqrt{T})
    +
    \tilde{O}(D^2d^{3/2}\sqrt{T})
    \\
    & =
    \tilde{O}(TD^2/d)
    +
    \tilde{O}(DT)
    +
    \tilde{O}(D^2d^{3/2}\sqrt{T}).
\end{align*}

We have
\begin{align*}
    I_1 & =
    \tilde{O}(\sqrt{d})
    \cdot 
    \sqrt{\tilde{O}(TD^2/d)
    +
    \tilde{O}(DT)
    +
    \tilde{O}(D^2d^{3/2}\sqrt{T})}
    \cdot
    \sqrt{\tilde{O}(d)}
    \\
    & = 
    \tilde{O}
    (D \sqrt{dT})
    +
    \tilde{O}
    (d \sqrt{DT})
    +
    \tilde{O}
    (D d^{7/4} T^{1/4}).
\end{align*}
Finally, by setting $\epsilon = 1/ \sqrt{T}$, the regret is upper bounded as
\begin{align*}
    \text{Regret}(T)
    &
    = \cO(\sqrt{T})
    +
    \tilde{O}
    (D \sqrt{dT})
    +
    \tilde{O}
    (d \sqrt{DT})
    +
    \tilde{O}
    (D d^{7/4} T^{1/4})
    +
    \tilde{O} (D \sqrt{T})
    +
    \tilde{O}(Dd)
    \\
    & = 
    \tilde{O}
    (D \sqrt{dT})
    +
    \tilde{O}
    (d \sqrt{DT})
    +
    \tilde{O}
    (D d^{7/4} T^{1/4}).
\end{align*}
As long as $d \ge D$ and $T \ge D^2d^3$, we have
\begin{align*}
    \text{Regret}(T)
    &
    = 
    \tilde{O}
    (d \sqrt{DT}).
\end{align*}
\end{proof}

\section{Proof of Theorem \ref{thm:lower-bound}} \label{sec:proof-lower-bound}
\subsection{Construction of Hard-to-learn MDPs} \label{subsec:construction}




  




Here we describe the construction of the hard-to-learn MDPs $M(\cS, \cA, \reward, \PP_{\btheta})$ for our lower bound proof (illustrated in Figure~\ref{fig:hardmdp}).  
The state space $\cS$ consists of two states $\state_0, \state_1$. The action space $\cA$ consists of $2^{d-1}$ vectors $\ab \in \RR^{d-1}$ whose coordinates are 1 or $-1$. The reward function $\reward$ satisfies that $\reward(\state_0, \ab) = 0$ and $\reward(\state_1, \ab) = 1$ for any $\ab \in \cA$. The probability transition function $\PP_{\btheta}$ is parameterized by a $(d-1)$-dimensional vector $\btheta \in \bTheta$, which is defined as
\begin{align}
    &\PP_{\btheta}(\state_0|\state_0, \ab) = 1-\delta - \la \ab, \btheta\ra,
    &\PP_{\btheta}(\state_1|\state_0, \ab) = \delta + \la \ab, \btheta\ra, &
    \notag\\
    &\PP_{\btheta}(\state_0|\state_1, \ab) = \delta, 
    &\PP_{\btheta}(\state_1|\state_1, \ab) = 1-\delta, &
    \notag\\
    &  \bTheta = \{ -\Delta/(d-1), \Delta/(d-1)\}^{d-1},\notag
\end{align}
where $\delta$ and $\Delta$ are positive parameters that need to be determined in later proof. We set $\delta = 1 / D$, and $\Delta $ as $\Delta = (1 / 45 \sqrt{2 \log 2 / 5}) d/\sqrt{DT}$.
It can be verified that $M$ is indeed a linear kernel MDP with the feature mapping $\bphi(s'|s,a)$ defined as follows:
\begin{align}
    & \bphi(\state_0|\state_0,\ab) = \begin{pmatrix}
    -\alpha\ab \\
    \beta(1-\delta)
    \end{pmatrix}, \bphi(\state_1|\state_0,\ab)=  \begin{pmatrix}
    \alpha\ab \\
    \beta\delta
    \end{pmatrix},
    \bphi(\state_0|\state_1,\ab)=  \begin{pmatrix}
    \zero \\
    \beta\delta
    \end{pmatrix}, 
    \bphi(\state_1|\state_1,\ab)=  \begin{pmatrix}
    \zero\\
    \beta(1-\delta)
    \end{pmatrix},\notag
\end{align}
where $\alpha = \sqrt{\Delta/[(d-1)(1+\Delta)]}$, $\beta = \sqrt{1/(1+\Delta)}$, and
the vector $\tilde\btheta = ( \btheta^\top/\alpha, 1/\beta)^\top \in \RR^d$. We can verify that $\bphi$ and $\tilde\btheta$ satisfy the requirements of $B$-bounded linear mixture MDP. In detail, we have
\begin{align}
    \|\tilde\btheta\|_2^2 = \frac{\|\btheta\|_2^2}{\alpha^2} + \frac{1}{\beta^2} = (1+\Delta)^2 \leq B^2,\notag
\end{align}
as long as $\Delta \leq \sqrt{B} - 1$. In addition, for any function $F: \cS \rightarrow [0,1]$, we have
\begin{align*}
    \|\bphi_F(\state_0, \ab)\|_2^2 
    &= 
    \alpha^2\|\ab\|_2^2(F(\state_1) - F(\state_0))^2
     + (\beta(1-\delta)F(\state_0) + \beta\delta F(\state_1))^2 \leq (d-1)\alpha^2 + \beta^2 = 1.
\end{align*}
Therefore, our defined MDP is indeed a $B$-bounded linear mixture MDP. 

\begin{remark}
Similar to \citet{zhou2020nearly}, our lower bound can also imply a lower bound for a related MDP class called \emph{linear MDPs} \citep{yang2019sample, jin2019provably}, which assumes that $\PP(s'|s,a) = \la \bpsi(s,a), \bmu(s')\ra$ and $\reward(s,a) = \la \bpsi(s,a), \bxi\ra$. We construct $\bpsi$, $\bmu$ and $\bxi$ as follows:
\begin{align*}
    \bpsi(s,a)& = \begin{cases}
    (\alpha\ab^\top, \beta, 0)^\top& s = \state_0\notag \\
    (0,0, 1)&s = \state_1\notag
    \end{cases}, 
    \bmu(s')  = 
    \begin{cases}
    (-\btheta^\top/\alpha, (1-\delta)/\beta), \delta)^\top & s' = \state_0\notag \\
    (\btheta^\top/\alpha, \delta/\beta, 1-\delta)^\top & s' = \state_1\notag
    \end{cases},
    \bxi  = (\zero^\top, 1)^\top.
\end{align*}
It can be verified that such a feature mapping $\bphi, \bmu$ and parameters $\bxi$ satisfy the requirements of a linear MDP with $(d+1)$-dimension feature mapping. Meanwhile, the MDP $\la \bpsi(s,a), \bmu(s')\ra$ has exactly the same form as the linear mixture MDPs proposed in Theorem \ref{thm:lower-bound}. Therefore,    the lower bound in Theorem \ref{thm:lower-bound} can also be applied to infinite-horizon average-reward linear MDPs, which are studied by \citet{wei2020learning}. This also suggests that there still exists a gap between the best upper bound \citep{wei2020learning} and lower bound in the linear MDP setting. 
\end{remark}

\subsection{Proof of the Lower Bound in Theorem~\ref{thm:lower-bound}}
Given the example we constructed above (shown in Figure~\ref{fig:hardmdp}), it is easy to see that the optimal policy is to choose the action $\ab$ satisfying $\la \ab, \btheta \ra = \Delta$, namely each coordinate of $\ab$ has the same sign as $\btheta$'s.

Given the optimal policy, it is clear that the stationary distribution is 
\begin{align*}
    \bmu 
    & = 
    \bigg[
    \frac{\delta}{2 \delta + \Delta}
    \qquad 
    \frac{\delta + \Delta}{2 \delta + \Delta}
    \bigg],
\end{align*}
and the optimal average reward is $\rho^* = {(\delta + \Delta)}/{(2 \delta + \Delta)}$.

In the construction, we leave the two parameters $\delta$ and $\Delta$ unspecified. Now we set $\delta = 1 / D$.
From state $x_1$ to $x_0$, it is clear that any policy has only one action and the expected travel time is $1/\delta = D$. From state $x_0$ to $x_1$, there always exists an policy that chooses the action $\ab$ that has the same sign coordinate-wise, and in that case the transition probability from $x_0$ to $x_1$ is $\delta + \Delta$, which indicates the expected travel time is smaller then $D$.
From the argument above, we know the MDP has a diameter of $D$.

The choice of $\Delta $ is $\Delta = (1 / 45 \sqrt{2 \log 2 / 5}) d/\sqrt{DT}$; the motivation will be revealed later in the proof.

In the following,  we use $\text{Regret}_{\btheta}(T)$ to denote the regret $\text{Regret}(M_{\btheta}, A, s, T)$, where $A$ is a deterministic algorithm. As argued in \citet{auer2002nonstochastic}, it is sufficient to only consider deterministic policies.
Let $\cP_{\btheta}(\cdot)$ denote the distribution over $\cS^T$, where $s_1 = \state_0$, $s_{t+1} \sim \PP_{\btheta}(\cdot|s_t, a_t)$, $a_t$ is decided by $A$. Let $\EE_{\btheta}$ denote the expectation w.r.t. distribution $\cP_{\btheta}$.
Denote $N_1, N_0, N_0^{\ab}$ as the random variables of the times state $x_1$ is visited, the times state $x_0$ is visited and the times state $x_0$ is visited and $\ab$ is chosen. 
We further define $N_0^{\cV}$ for some subset $\cV \subset \cA$ as the random variable of the times state $x_0$ is visited, and the action $\ab$ belongs to $\cV$.

\begin{lemma}\label{lemma:n1}
Suppose $2\Delta<\delta$ and $(1-\delta)/\delta < T/5$, then for $\EE_{\btheta}N_1$ and $\EE_{\btheta}N_0$, we have
\begin{align}
    &\EE_{\btheta}N_1 \leq \frac{T}{2} +\frac{1}{2\delta}\sum_{\ab}\la \ab, \btheta\ra \EE_{\btheta}N_0^{\ab},\ \EE_{\btheta}N_0 \leq 4T/5.\notag
\end{align}
\end{lemma}
\begin{proof}
See Section~\ref{sec:proof-lemma-n1}.
\end{proof}

\begin{lemma}[Pinsker's inequality, in \citet{jaksch2010near}]\label{lemma:pinsker}
Denote $\sbb = \{s_1,\dots,s_T\} \in \cS^{T}$ as the observed states from step $1$ to $T$. Then for any two distributions $\cP_1$ and $\cP_2$ over $\cS^{T}$ and any bounded function $f: \cS^{T}\rightarrow [0, B]$, we have
\begin{align}
    \EE_1 f(\sbb) - \EE_2 f(\sbb) \leq \sqrt{\log 2/2}B\sqrt{\text{KL}(\cP_2\|\cP_1)},\notag
\end{align}
where $\EE_1$ and $\EE_2$ are expectations with respect to $\cP_1$ and $\cP_2$. 
\end{lemma}

\begin{lemma}\label{lemma:kl}
Suppose that $\btheta$ and $\btheta'$ only differs from $j$-th coordinate, $2\Delta<\delta \leq 1/3 $. Then we have the following bound for the KL divergence between $\cP_{\btheta}$ and $\cP_{\btheta'}$:
\begin{align}
    \text{KL}(\cP_{\btheta'}\| \cP_{\btheta}) \leq \frac{16\Delta^2}{(d-1)^2\delta}\EE_{\btheta}N_0.\notag
\end{align}
\end{lemma}
\begin{proof}
See Section~\ref{sec:proof-lemma-kl}.
\end{proof}

\begin{proof}[Proof of Theorem \ref{thm:lower-bound}]
We have
\begin{align*}
    \EE_{\btheta}
    [
    \text{Regret}_{\btheta}(T)
    ]
    & :=
    T \rho^*
    -
    \EE_{\btheta} \bigg[
    \sum_{t=1}^{T}
    r(s_t, a_t)
    \bigg]
    \\
    & = 
    T \rho^*
    -
    \EE_{\btheta} [N_1].
\end{align*}
Averaging over all possible choice of $\btheta \in \bTheta$,
we have 
\begin{align*}
    \frac{1}{|\bTheta|}
    \sum_{\btheta} 
    \EE_{\btheta}
    [
    \text{Regret}_{\btheta}(T)
    ]
    & =
    T \rho^*
    -
    \frac{1}{|\bTheta|}
    \sum_{\btheta} 
    \EE_{\btheta} [N_1].
\end{align*}
Following Lemma~\ref{lemma:n1}, we first have
\begin{align} \label{eqn:N1}
    \frac{1}{|\bTheta|}
    \sum_{\btheta} 
    \EE_{\btheta} [N_1]
    & \le 
    \frac{T}{2} 
    +
    \frac{1}{2 \delta|\bTheta|}
    \sum_{\btheta} 
    \sum_{\ab}
    \la \ab, \btheta \ra 
    \EE_{\btheta} N_{0}^{\ab} 
    \notag \\ 
    & =
    \frac{T}{2} 
    +
    \frac{1}{2 \delta|\bTheta|}
    \sum_{\btheta} 
    \sum_{\ab}
    \frac{\Delta}{d-1}
    \sum_{j=1}^{d-1}
    \ind \{ 
    \text{sign} (\ab_j) = \text{sign}(\btheta_j)
    \}
    \EE_{\btheta} N_{0}^{\ab} 
    \notag \\ 
    & =
    \frac{T}{2} 
    +
    \frac{1}{2 \delta|\bTheta|}
    \frac{\Delta}{d-1}
    \sum_{j=1}^{d-1}
    \sum_{\btheta} 
    \sum_{\ab}
    \EE_{\btheta} 
    \big[
    \ind \{ 
    \text{sign} (\ab_j) = \text{sign}(\btheta_j)
    \}
    N_{0}^{\ab} 
    \big].
\end{align}
For a fixed coordinate $j$, consider $\btheta'$ that only differs with $\btheta$ at its $j$-th coordinate. We have 
\begin{align*}
    &\EE_{\btheta} 
    \big[
    \ind \{ 
    \text{sign} (\ab_j) = \text{sign}(\btheta_j)
    \}
    N_{0}^{\ab} 
    \big]
    +
    \EE_{\btheta'} 
    \big[
    \ind \{ 
    \text{sign} (\ab_j) = \text{sign}(\btheta'_j)
    \}
    N_{0}^{\ab} 
    \big]
    \\ 
    & =
    \EE_{\btheta'} 
    \big[
    N_{0}^{\ab} 
    \big]
    +
    \EE_{\btheta} 
    \big[
    \ind \{ 
    \text{sign} (\ab_j) = \text{sign}(\btheta_j)
    \}
    N_{0}^{\ab} 
    \big]
    -
    \EE_{\btheta'} 
    \big[
    \ind \{ 
    \text{sign} (\ab_j) = \text{sign}(\btheta_j)
    \}
    N_{0}^{\ab} 
    \big],
\end{align*}
since $\ind \{ 
    \text{sign} (\ab_j) = \text{sign}(\btheta'_j)
    \} = 1-\ind \{ 
    \text{sign} (\ab_j) \ne  \text{sign}(\btheta_j)
    \}$.
    
Summing the equation above over $\bTheta$ and $\cA$, we have 
\begin{align} \label{eqn:N0a}
    & 2\sum_{\btheta} 
    \sum_{\ab} 
    \EE_{\btheta} 
    \big[
    \ind \{ 
    \text{sign} (\ab_j) = \text{sign}(\btheta_j)
    \}
    N_{0}^{\ab} 
    \big]
    \notag \\
    & = 
    \sum_{\btheta} 
    \sum_{\ab} 
    \EE_{\btheta'} 
    \big[
    N_{0}^{\ab} 
    \big]
    +
    \sum_{\btheta} 
    \bigg [
    \EE_{\btheta} 
    \Big[
    \sum_{\ab} 
    \ind \{ 
    \text{sign} (\ab_j) = \text{sign}(\btheta_j)
    \}
    N_{0}^{\ab} 
    \Big]
    -
    \EE_{\btheta'} 
    \Big[
    \sum_{\ab} 
    \ind \{ 
    \text{sign} (\ab_j) = \text{sign}(\btheta_j)
    \}
    N_{0}^{\ab} 
    \Big]
    \bigg]
    \notag \\ 
    & =
    \sum_{\btheta} 
    \EE_{\btheta'} 
    \big[
    N_{0} 
    \big]
    +
    \sum_{\btheta} 
    \bigg [
    \EE_{\btheta} 
    \Big[ 
    N_{0}^{\cA_j} 
    \Big]
    -
    \EE_{\btheta'} 
    \Big[
    N_{0}^{\cA_j} 
    \Big]
    \bigg]
    \notag \\ 
    & \le 
    \sum_{\btheta} 
    \EE_{\btheta'} 
    \big[
    N_{0} 
    \big]
    +
    \sum_{\btheta} 
    \sqrt{\log 2 / 2} 
    T
    \sqrt{\text{KL}(\cP_{\btheta'} \| \cP_{\btheta} )}
    \notag \\ 
    & \le 
    \sum_{\btheta} 
    \EE_{\btheta'} 
    \big[
    N_{0} 
    \big]
    +
    \sum_{\btheta} 
    2\sqrt{2\log 2} 
    \frac{T \Delta}{d \sqrt{\delta}}
    \sqrt{\EE_{\btheta} [N_0]}, 
\end{align}
where $\cA_j$ is the set of all $\ab$ which satisfy $\ind \{ \text{sign}(\ab_j)=\text{sign}(\btheta_j)\}$. 
The first equality is by matching each $\btheta$ with $\btheta'$ that differs from $\btheta$ in its $j$-th coordinate, and moving $\sum_{\ab}$ inside. The second equality applies the shorthand $\cA_j$.  The first inequality is due to Lemma~\ref{lemma:pinsker}. The last is due to Lemma~\ref{lemma:kl}.

Substituting \eqref{eqn:N0a} into \eqref{eqn:N1}, we have
\begin{align*}
    \frac{1}{|\bTheta|}
    \sum_{\btheta} 
    \EE_{\btheta} [N_1]
    & \le 
    \frac{T}{2} 
    +
    \frac{1}{4 \delta|\bTheta|}
    \frac{\Delta}{d-1}
    \sum_{j=1}^{d-1}
    \sum_{\btheta}
    \bigg[ 
    \EE_{\btheta'} 
    \big[
    N_{0} 
    \big]
    +
    2\sqrt{2\log 2} 
    \frac{T \Delta}{d \sqrt{\delta}}
    \sqrt{\EE_{\btheta} [N_0]}
    \bigg] 
    \\
    & = 
    \frac{T}{2} 
    +
    \frac{\Delta}{4 \delta|\bTheta|}
    \sum_{\btheta}
    \bigg[ 
    \EE_{\btheta'} 
    \big[
    N_{0} 
    \big]
    +
    2\sqrt{2\log 2} 
    \frac{T \Delta}{d \sqrt{\delta}}
    \sqrt{\EE_{\btheta} [N_0]}
    \bigg] 
    \\ 
    & \le 
    \frac{T}{2} 
    +
    \frac{\Delta}{4 \delta}
    \bigg[ 
    \frac{4T}{5}
    +
    2\sqrt{2\log 2} 
    \frac{T \Delta}{d \sqrt{\delta}}
    \frac{2 \sqrt{T}}{\sqrt{5}}
    \bigg] 
    \\ 
    & =
    \frac{T}{2}
    +
    \frac{\Delta T}{5 \delta}
    +
    \sqrt{2 \log 2 / 5}
    \frac{\Delta^2 T^{3/2}}{d \delta^{3/2}},
\end{align*}
where the second inequality is due to Lemma~\ref{lemma:n1}.

This further leads to
\begin{align*}
    \frac{1}{|\bTheta|}
    \sum_{\btheta} 
    \EE_{\btheta}
    [
    \text{Regret}_{\btheta}(T)
    ]
    & =
    T \rho^*
    -
    \frac{1}{|\bTheta|}
    \sum_{\btheta} 
    \EE_{\btheta} [N_1]
    \\
    & \ge  
    T \cdot \frac{\delta + \Delta}{2 \delta + \Delta} 
    -
    \frac{T}{2}
    -
    \frac{\Delta T}{5 \delta}
    -
    \sqrt{2 \log 2 / 5}
    \frac{\Delta^2 T^{3/2}}{d \delta^{3/2}}
    \\
    & = 
    \frac{\Delta (\delta - 2 \Delta)}{5 \delta (4\delta + 2 \Delta)}
    \cdot T
    -
    \sqrt{2 \log 2 / 5}
    \frac{\Delta^2 T^{3/2}}{d \delta^{3/2}}
    \\ 
    & \ge 
    \frac{2}{45 \delta}
    \cdot \Delta T
    -
    \sqrt{2 \log 2 / 5} 
    \cdot \frac{\Delta^2 T^{3/2}}{d \delta^{3/2}}
    \\ 
    & =  
    \frac{1}{2025 \sqrt{2 \log 2 / 5}}
    \cdot 
    d \sqrt{DT}
    \\
    & >
    \frac{1}{2025}
    \cdot 
    d \sqrt{DT},
\end{align*}
where the second inequality requires $0 < 4\Delta \le \delta$; the last equality is due to the setting $\delta = D^{-1}$ and $\Delta = (1 / 45 \sqrt{2 \log 2 / 5}) d/\sqrt{DT}$.
This further requires that $T \ge 16 d^2 D / 2025$.
\end{proof}

\section{Proof of Supporting Lemmas}
\subsection{Proof of Lemma \ref{lemma:theta-ball}} \label{sec:proof-lemma-theta-ball}
\begin{proof}[Proof of Lemma \ref{lemma:theta-ball}]
Recall the definition of $\btheta_k$ in Algorithm \ref{algorithm}, we have
\begin{align}
    \btheta_{k}  = \bigg(\lambda \Ib+\sum_{j=0}^{k-1}\sum_{i=t_j}^{t_{j+1}-1}\bphi_{{w}_{j}}(s_i,a_i)\bphi_{{w}_{j}}(s_i,a_i)^\top\bigg)^{-1}\bigg(\sum_{j=0}^{k-1}\sum_{i=t_j}^{t_{j+1}-1} \bphi_{{w}_{j}}(s_i,a_i) w(s_{i+1})\bigg).\notag
\end{align}
It is worth noting that for any $0 \leq j \leq k-1$ and $t_j \leq i \leq t_{j+1}-1$,
\begin{align}
    [\PP w_j](s_i, a_i) &= \int_{s'}\PP(s'|s_i, a_i) w_j(s_i, a_i)ds' \notag \\
    &= \int_{s'}\la \bphi(s'|s_i,a_i), \btheta^*\ra w_j(s')ds'\notag \\
    & = \Big\la\int_{s'} \bphi(s'|s_i,a_i) w_j(s'), \btheta^*\Big\ra\notag \\
    & = \la \bphi_{{w}_{j}}(s_i,a_i), \btheta^*\ra,
\end{align}
thus $\{w_j(s_{i+1}) - \la \bphi_{w_j}(s_i,a_i), \btheta^*\ra\}$ forms a martingale difference sequence. Besides, since $|w(s)| \leq D/2$ for any $s$, then $w_j(s_{i+1}) - \la \bphi_{w_j}(s_i,a_i), \btheta^*\ra$ is a sequence of $D$-subgaussian random variables with zero means. Meanwhile, we have $\|\bphi_{w_j}(s_i,a_i)\|_2 \leq D$ and $\|\btheta^*\|_2 \leq B$ by Definition \ref{assumption-linear}. 
By Theorem 2 in \citet{abbasi2011improved}, we have that with probability at least $1-\delta$, $\btheta^*$ belongs to the following set for all $1 \leq k \leq K$:
\begin{align}
    \bigg\{\btheta: \Big\|\bSigma_{t_k}^{1/2}(\btheta - \hat\btheta_k)\Big\|_2 \leq D \sqrt{\log \bigg( \frac{\lambda +t_k  D^2}{\delta\lambda} \bigg)} + \sqrt{\lambda} B \bigg\}.
\end{align}
Finally, by the definition of $\hbeta_t$ and the fact that $\la \btheta^*, \bphi(s'|s,a)\ra = \PP(s'|s,a)$ for all $(s,a)$, we draw the conclusion that $\btheta^* \in \hat{\cC}_{t_k}$ for $1 \leq k \leq K$. 
\end{proof}

\subsection{Proof of Lemma \ref{lemma:boundk}} \label{sec:proof-lemma-boundk-alg1}
\begin{proof}[Proof of Lemma \ref{lemma:boundk}]
For simplicity, we denote $K = K(T)$. 
Note that $\det(\bSigma_0) = \lambda^d$. We further have
\begin{align}
    \|\bSigma_T\|_2 &= \bigg\|\lambda\Ib+\sum_{k=0}^{K-1}\sum_{t=t_k}^{t_{k+1}-1}\bphi_{w_k}(s_t, a_t)\bphi_{w_k}(s_t, a_t)^\top\bigg\|_2 \notag \\
    & \leq \lambda + \sum_{k=0}^{K-1}\sum_{t=t_k}^{t_{k+1}-1}\big\|\bphi_{w_k}(s_t, a_t)\big\|_2^2\notag \\
    & \leq \lambda + T D^2, \label{eq:boundk_0}
\end{align}
where the first inequality holds due to the triangle inequality, the second inequality holds due to the fact $w_k(s) \leq D/2$ and Definition \ref{assumption-linear}. \eqref{eq:boundk_0} suggests that $\det(\bSigma_T) \leq (\lambda+Td D^2)^d$. Therefore, we have
\begin{align}
    (\lambda + T D^2 )^d \geq \det(\bSigma_T) \geq \det(\bSigma_{t_{K-1}-1})\geq 2^{K-1}\det(\bSigma_{t_{0}-1}) = 2^{K-1}\lambda^d,  \label{eq:boundk_1}
\end{align}
where the second inequality holds since $\bSigma_T\succeq\bSigma_{t_{K-1}-1} $, the third inequality holds due to the fact that $\det(\bSigma_{t_{k}-1}) \geq 2\det(\bSigma_{t_{k-1}-1})$ by the update rule in Algorithm \ref{algorithm}. \eqref{eq:boundk_1} suggests 
\begin{align}
    K \leq d\log \frac{2\lambda +2T D^2}{\lambda}.\notag
\end{align}
\end{proof}

\subsection{Proof of Lemma~\ref{lemma:bernstein-boundk}} \label{sec:proof-bernstein-boundk}
\begin{proof}[Proof of Lemma~\ref{lemma:bernstein-boundk}]
For simplicity, we denote $K = K(T)$. 
Note that $\det(\hbSigma_1) = \lambda^d$. We further have
\begin{align*}
    \|\hbSigma_{t_K}\|_2 &= \bigg\|\lambda\Ib+
    \sum_{t=1}^{T}
    \bphi_{w_k}(s_t, a_t)
    \bphi_{w_k}(s_t,a_t)^\top
    / \barsigma_t^2
    \bigg\|_2 
    \leq 
    \lambda + \sum_{t=1}^{T}
    \big\|\bphi_{w_k}(s_t, a_t) / \barsigma_t\big\|_2^2
    \leq \lambda + T d,
\end{align*}
where the first inequality holds due to the triangle inequality, the second inequality holds because $w_k(s) \leq D$ and $\barsigma_t \ge D/\sqrt{d}$. This suggests that $\det(\hbSigma_{t_K}) \leq (\lambda+Td)^d$. Therefore, we have
\begin{align*}
    (\lambda + T d )^d \geq \det(\bSigma_{t_K}) \geq \det(\bSigma_{t_{K-1}})\geq 2^{K-1}\det(\bSigma_{t_{0}}) = 2^{K-1}\lambda^d,
\end{align*}
where the second inequality holds since $\bSigma_T\succeq\bSigma_{t_{K-1}-1} $, the third inequality holds due to the fact that $\det(\bSigma_{t_{k}-1}) \geq 2\det(\bSigma_{t_{k-1}-1})$ by the update rule in Algorithm \ref{algorithm} with OPTION 2. This suggests 
\begin{align*}
    K \leq 2d\log (1 + dT/\lambda).
\end{align*}
\end{proof}

\subsection{Proof of Lemma~\ref{lemma:confidence-set}} \label{sec:proof-lemma-confidence-set}
\begin{proof}[Proof of Lemma~\ref{lemma:confidence-set}]
In fact we are able to prove a stronger result:
\begin{align*}
    \btheta^* \in \hat{\cC}_t
    \cap \check{\cC}_t
    \cap \tilde{\cC}_t
    \cap \cB,
\end{align*}
where the two additional sets are defined as
\begin{align*}
	    \check{\cC}_t & := \bigg\{ 
	    \btheta : \Big \| 
	     \cbSigma_{t}^{1/2} (\btheta - \cbtheta_t)
	    \Big \| \le \cbeta_t
	    \bigg\}
	    \\
	    \tilde{\cC}_t & := \bigg\{ 
	    \btheta : \Big \| 
	     \tbSigma_{t}^{1/2} (\btheta - \tbtheta_t)
	    \Big \| \le \tbeta_t
	    \bigg\}.
\end{align*}
For any $1 \le t \le T$, we always have $k$ such that $t_k \le t < t_{k+1}$.
We start with the following inequality:
\begin{align*}
    \big | [\bar{\VV}_t w_k](s_t,a_t)
    -
    [\VV w_k](s_t,a_t) \big |
    & =
    \bigg|
    \min \Big \{ 
    D^2/4, \big \la \bphi_{w_k^2}(s_t, a_t), \tbtheta_t  \big \ra
    \Big \}
    -
    \big \la \bphi_{w_k^2}(s_t, a_t), \btheta^*  \big \ra
    \\
    & \qquad
    +
    \big \la \bphi_{w_k}(s_t, a_t), \btheta^*  \big \ra^2
    -
    \Big[ 
    \min \Big \{ 
    D/2, \big \la \bphi_{w_k}(s_t, a_t), \btheta_t  \big \ra
    \Big \}
    \Big]^2
    \bigg |
    \\
    & \le 
    \underbrace{
    \bigg|
    \min \Big \{ 
    D^2/4, \big \la \bphi_{w_k^2}(s_t, a_t), \tbtheta_t  \big \ra
    \Big \}
    -
    \big \la \bphi_{w_k^2}(s_t, a_t), \btheta^*  \big \ra
    \bigg |}_{I_1}
    \\
    & \qquad
    +
    \underbrace{
    \bigg |
    \big \la \bphi_{w_k}(s_t, a_t), \btheta^*  \big \ra^2
    -
    \Big[ 
    \min \Big \{ 
    D/2, \big \la \bphi_{w_k}(s_t, a_t), \btheta_t  \big \ra
    \Big \}
    \Big]^2
    \bigg |}_{I_2},
\end{align*}
where the inequality is by the triangle inequality.

For $I_1$, we have
\begin{align*}
    I_1
    & \le 
    \Big|
    \big \la \bphi_{w_k^2}(s_t, a_t), \tbtheta_t  \big \ra
    -
    \big \la \bphi_{w_k^2}(s_t, a_t), \btheta^*  \big \ra
    \Big |
    \\
    & =
    \Big|
    \big \la \bphi_{w_k^2}(s_t, a_t), \tbtheta_t - \btheta^*  \big \ra
    \Big |
    \\
    & \le 
    \Big\|
    \tbSigma_t^{-1/2}
    \bphi_{w_k^2}(s_t, a_t)
    \Big\|
    \cdot
    \Big\|
    \tbSigma_t^{1/2}
    (\tbtheta_t - \btheta^*)  
    \Big \|,
\end{align*}
where the first inequality is due to $\big \la \bphi_{w_k^2}(s_t, a_t), \btheta^*  \big \ra = \EE_{s' \sim \PP(\cdot|s_t,a_t)}[ w_k^2(s')] \in [0, D^2/4]$, and the last inequality is due to Cauchy-Schwartz inequality. Also, it is clear $I_1 \le D^2/4$.

Similarly, for $I_2$, we have
\begin{align*}
    I_2
    & =
    \bigg |
    \big \la \bphi_{w_k}(s_t, a_t), \btheta^*  \big \ra
    +
    \min \Big \{ 
    D/2, \big \la \bphi_{w_k}(s_t, a_t), \btheta_t  \big \ra
    \Big \}
    \bigg |
    \\
    & \qquad
    \cdot 
    \bigg |
    \big \la \bphi_{w_k}(s_t, a_t), \btheta^*  \big \ra
    -
    \Big[ 
    \min \Big \{ 
    D/2, \big \la \bphi_{w_k}(s_t, a_t), \btheta_t  \big \ra
    \Big \}
    \Big]
    \bigg |
    \\ 
    & \le 
    D \cdot \Big |
    \big \la \bphi_{w_k}(s_t, a_t), \btheta^*  \big \ra
    -\big \la \bphi_{w_k}(s_t, a_t), \btheta_t  \big \ra
    \Big |
    \\
    & =
    D \cdot \Big |
    \big \la \bphi_{w_k}(s_t, a_t), \btheta^*  
    - \btheta_t  \big \ra
    \Big | 
    \\
    & \le 
    D \cdot \Big \|
    \hbSigma_t^{-1/2} \bphi_{w_k}(s_t, a_t)
    \Big \|
    \cdot 
    \Big \|
    \hbSigma_t^{1/2}(
    \btheta^*  
    - \btheta_t)  
    \Big \| ,
\end{align*}
where the first equality is by $a^2-b^2 = (a+b)(a-b)$, and the following reasoning is the same as $I_1$. The only additional fact used in the first inequality is $\big \la \bphi_{w_k}(s_t, a_t), \btheta^*  \big \ra \in [0,D/2]$ and $\min \big \{ 
    D/2, \big \la \bphi_{w_k}(s_t, a_t), \btheta_t  \big \ra
    \big \} \in [0,D/2]$. Also, it is clear $I_2 \le D^2/4$.
    
The two terms combined together gives
\begin{align}
    \big | [\bar{\VV}_t w_k](s_t,a_t)
    -
    [\VV w_k](s_t,a_t) \big |
    &
    \le 
    \min \Big \{
    D^2/4
    ,
    \Big\|
    \tbSigma_t^{-1/2}
    \bphi_{w_k^2}(s_t, a_t)
    \Big\|
    \cdot
    \Big\|
    \tbSigma_t^{1/2}
    (\tbtheta_t - \btheta^*)  
    \Big \|
    \Big \}
    \notag \\
    & \qquad +
    \min \Big \{
    D^2/4
    ,
    D \cdot \Big \|
    \hbSigma_t^{-1/2} \bphi_{w_k}(s_t, a_t)
    \Big \|
    \cdot 
    \Big \|
    \hbSigma_t^{1/2}(
    \btheta^*  
    - \btheta_t)  
    \Big \|
    \Big \}. \label{eqn:variance-bound}
\end{align}
Now, we first show that with probability $1-\delta$, for all $t$, $\btheta^* \in \check{\cC}_t$. To show this, we apply Lemma~\ref{lemma:concentration_variance}. By setting $\xb_t = \barsigma_t^{-1} \bphi_{w_k}(s_t,a_t)$ and $\eta_t = \barsigma_t^{-1} w_k(s_{t+1}) - \barsigma_t^{-1} \la \bphi_{w_k}(s_t,a_t), \btheta^* \ra$, $\cG_t = \cF_t$, $\bmu^* = \btheta^*$, $y_t = \la \bmu^*, \xb_t \ra + \eta_t$, $\Zb_t = \lambda \Ib + \sum_{i=1}^{t} \xb_i \xb_i^{\top}$, $\bbb_t = \sum_{i=1}^{t} \xb_i y_i$ and $\bmu_t = \Zb_t^{-1} \bbb_t$, we have $y_t = \barsigma_t^{-1} w_k(s_{t+1})$ and $\bmu_t = \hbtheta_t$. Moreover, we have
\begin{align*}
    \| \xb_t \|_2 \le \sqrt{d}/2,
    | \eta_t | \le \sqrt{d},
    \EE[\eta_t | \cG_t] = 0,
    \EE[\eta_t^2 | \cG_t] = d.
\end{align*}
Therefore, by Lemma~\ref{lemma:concentration_variance}, we have with probability $1- \delta$, for all $t \in [T]$,
\begin{align*}
    \| \hbSigma_t^{1/2} (\hbtheta_t - \btheta^*)
    \|_2 \le 
    8 d \sqrt{ \log(1 + t/4 \lambda) \log(4t^2/\delta)}
    +
    4 \sqrt{d} \log(4t^2/ \delta)
    +
    \sqrt{\lambda} B = \cbeta_t.
\end{align*}
This means that with probability $1-\delta$, for all $t$, $\btheta^* \in \check{\cC}_t$.

The same argument can be applied again, except that now we focus on the squared function $w_k^2$. This gives
\begin{align*}
    \| \tbSigma_t^{1/2} (\tbtheta_t - \btheta^*)
    \|_2 \le 
    8 (D^2/4) \sqrt{d \log(1 + tD^2/4 \lambda d \lambda) \log(4t^2/\delta)}
    +
    4(D^2/4) \log(4t^2/ \delta)
    +
    \sqrt{\lambda} B = \tbeta_t.
\end{align*}
This means that with probability $1-\delta$, for all $t$, $\btheta^* \in \tilde{\cC}_t$.

Now we show that $\btheta^* \in \hat{\cC}_t$ with high probability. Let $\xb_t = \barsigma_t^{-1} \bphi_{w_k}(s_t,a_t)$, and 
\begin{align*}
    \eta_t = \barsigma_t^{-1}
    \ind \{ \btheta^* \in \check{\cC}_t \cap \tilde{\cC}_t \}
    \big[w_k(s_{t+1}) -  \la \bphi_{w_k}(s_t,a_t)
    , \btheta^* \ra
    \big].
\end{align*}
In this case, it is clear that still we have $\EE[\eta_t| \cG_t] = 0$, $| \eta_t | \le \sqrt{d}$, $\| \xb_t \|_2 \le \sqrt{d}$. Also,
\begin{align}
    \EE[\eta_t^2|\cG_t] & = \barsigma_t^{-2}
    \ind \{ \btheta^* \in \check{\cC}_t \cap \tilde{\cC}_t \}
    [\VV w_{t}](s_t, a_t)\notag \\
    & \le 
    \barsigma_t^{-2}
    \ind \{ \btheta^* \in \check{\cC}_t \cap \tilde{\cC}_t \}
    \bigg[
    [\bar{\VV}_t w_{t}](s_t, a_t) \notag \\
    &\qquad + 
    \min \Big \{
    D^2/4
    ,
    \Big\|
    \tbSigma_t^{-1/2}
    \bphi_{w_k^2}(s_t, a_t)
    \Big\|
    \cdot
    \Big\|
    \tbSigma_t^{1/2}
    (\tbtheta_t - \btheta^*)  
    \Big \|
    \Big \}
    \notag \\
    & \qquad +
    \min \Big \{
    D^2/4
    ,
    D \cdot \Big \|
    \hbSigma_t^{-1/2} \bphi_{w_k}(s_t, a_t)
    \Big \|
    \cdot 
    \Big \|
    \hbSigma_t^{1/2}(
    \btheta^*  
    - \btheta_t)  
    \Big \|
    \Big \} 
    \bigg]
    \notag \\
    & \le
    \barsigma_t^{-2}
    \bigg[
    [\bar{\VV}_t w_{t}](s_t, a_t) 
    +
    \min \Big \{
    D^2/4
    ,
    \Big\|
    \tbSigma_t^{-1/2}
    \bphi_{w_k^2}(s_t, a_t)
    \Big\|
    \tbeta_t
    \Big \}
    \notag \\
    & \qquad +
    \min \Big \{
    D^2/4
    ,
    D \cbeta_t \cdot \Big \|
    \hbSigma_t^{-1/2} \bphi_{w_k}(s_t, a_t)
    \Big \|
    \Big \} 
    \bigg]
    \notag \\
    & = 1, \notag
\end{align}
where the first inequality is due to \eqref{eqn:variance-bound} and  the second inequality is due to first, the event that $\btheta^* \in \check{\cC}_t \cap \tilde{\cC}_t$. The last equality is by the definition of $\barsigma_t$.

Again by Lemma~\ref{lemma:concentration_variance}, we have that for all $t \in [T]$, 
\begin{align*}
    \| \bmu_t - \bmu^* \|_{\Zb_t}
    \le 
    8 \sqrt{ d \log(1 + t/4 \lambda) \log(4t^2/\delta)}
    +
    4 \sqrt{d} \log(4t^2/ \delta)
    +
    \sqrt{\lambda} B = \hbeta_t.
\end{align*}
Now, denote the event when $\big\{ \forall t \in [T], \btheta^* \in \check{\cC}_t \cap \tilde{\cC}_t  \big\}$ and the inequality above holds as $\cE_0$. By union bound, we have $\PP(\cE_0) \ge 1 - 3\delta$.

It is clear that under $\cE_0$, we have $\btheta^* \in \hat{\cC}_t$ for all $t$ because under event $\cE_0$,
\begin{align*}
    y_t & = \la \barsigma_t^{-1} \bphi_{w_k}(s_t,a_t), \btheta^* \ra 
    +
    \barsigma_t^{-1}
    \ind \{ \btheta^* \in \check{\cC}_t \cap \tilde{\cC}_t \}
    \big[w_k(s_{t+1}) -  \la \bphi_{w_k}(s_t,a_t)
    , \btheta^* \ra
    \big]
    \\ 
    & = \barsigma_t^{-1} w_k(s_{t+1}),
\end{align*}
so indeed we have $\| \hbtheta_t - \btheta^* \|_{\hbSigma_t} \le \hbeta_t$.

Also, by the definition of $E_t$, it is clear that under event $\cE_0$,
\begin{align*}
    \big | [\bar{\VV}_t w_k](s_t,a_t)
    -
    [\VV w_k](s_t,a_t) \big |
    \le 
    E_t.
\end{align*}

\end{proof}

\subsection{Proof of Lemma~\ref{lemma:total-variance}} \label{sec:proof-lemma-total-variance}
\begin{proof}[Proof of Lemma~\ref{lemma:total-variance}]
Part of the proof is inspired by \citet{fruit2020improved}. 
We will use $\VV_P(w)$ to denote $\EE_{s' \sim P(\cdot)}[w(s')^2] - (\EE_{s' \sim P(\cdot)}[w(s')])^2$, namely the variance of the random variable $w(s')$ where $s' \sim P(\cdot)$. Some examples are
\begin{align*}
    \VV_{\PP(\cdot | s_t, a_t)}(w_k) & = 
    \EE_{s' \sim \PP(\cdot | s_t, a_t)}[w_k(s')^2] - (\EE_{s' \sim \PP(\cdot | s_t, a_t)}[w_k(s')])^2,
    \\
    \VV_{\PP_k(\cdot | s_t, a_t)}(w_k) & = 
    \EE_{s' \sim \PP_k(\cdot | s_t, a_t)}[w_k(s')^2] - (\EE_{s' \sim \PP_k(\cdot | s_t, a_t)}[w_k(s')])^2.
\end{align*}
When the context is clear, we may also use short-hands like $\EE_{p}[w(s')]$ to indicate expectation under $p(\cdot)$.

The following decomposition is useful:
\begin{align*}
    & \sum_{k=0}^{K(T)-1}
    \sum_{t = t_k}^{t_{k+1} - 1} 
    \VV_{\PP(\cdot | s_t, a_t)}(w_k) 
    \\
    & = 
    \sum_{k=0}^{K(T)-1}
    \sum_{t = t_k}^{t_{k+1} - 1} 
    \EE_{s' \sim \PP(\cdot | s_t, a_t)}[w_k(s')^2] - (\EE_{s' \sim \PP(\cdot | s_t, a_t)}[w_k(s')])^2
    \\
    & = 
    \sum_{k=0}^{K(T)-1}
    \sum_{t = t_k}^{t_{k+1} - 1} 
    \big[
    \EE_{s' \sim \PP(\cdot | s_t, a_t)}[w_k(s')^2] 
    - w_k^2(s_{t+1})
    \big] 
    \\
    & \qquad +
    \sum_{k=0}^{K(T)-1}
    \sum_{t = t_k}^{t_{k+1} - 1} 
    \big[
    w_k^2(s_{t+1})
    -
    (\EE_{s' \sim \PP(\cdot | s_t, a_t)}[w_k(s')])^2
    \big]
    \\
    & = 
    \sum_{k=0}^{K(T)-1}
    \sum_{t = t_k}^{t_{k+1} - 1} 
    \big[
    \EE_{s' \sim \PP(\cdot | s_t, a_t)}[w_k(s')^2] 
    - w_k^2(s_{t+1})
    \big] 
    \\
    & \qquad +
    \sum_{k=0}^{K(T)-1}
    \Bigg[
    \sum_{t = t_k}^{t_{k+1} - 1} 
    \big[
    w_k^2(s_{t})
    -
    (\EE_{s' \sim \PP(\cdot | s_t, a_t)}[w_k(s')])^2
    \big]
    + w_k^2(s_{t_{k+1}})
    -
    w_k^2(s_{t_k})
    \Bigg]
    \\
    & \le 
    \underbrace{
    \sum_{k=0}^{K(T)-1}
    \sum_{t = t_k}^{t_{k+1} - 1} 
    \big[
    \EE_{s' \sim \PP(\cdot | s_t, a_t)}[w_k(s')^2] 
    - w_k^2(s_{t+1})
    \big] }_{I_1}
    \\
    & \qquad +
    \underbrace{
    \sum_{k=0}^{K(T)-1}
    \sum_{t = t_k}^{t_{k+1} - 1} 
    \big[
    w_k^2(s_{t})
    -
    (\EE_{s' \sim \PP(\cdot | s_t, a_t)}[w_k(s')])^2
    \big]}_{I_2}
    +
    K(T) \cdot D^2/4.
\end{align*}
For term $I_1$, since the event $\cE_1$ holds, we have 
\begin{align*}
    I_1 \le (D^2/4) \sqrt{2T \log(1/\delta)}.
\end{align*}
For term $I_2$, we have
\begin{align*}
    I_2
    & = 
    \sum_{t=1}^{T}
    \big[
    w_k^2(s_t)
    -
    (\EE_{s' \sim \PP(\cdot | s_t, a_t)}[w_k(s')])^2
    \big]
    \\
    & \le 
    \sum_{t=1}^{T}
    \big| 
    w_k(s_t)
    -
    \EE_{s' \sim \PP(\cdot | s_t, a_t)}[w_k(s')]
    \big|
    \cdot
    \big| 
    w_k(s_t)
    +
    \EE_{s' \sim \PP(\cdot | s_t, a_t)}[w_k(s')]
    \big|
    \\
    & \le 
    D
    \sum_{t=1}^{T}
    \big| 
    w_k(s_t)
    -
    (\EE_{s' \sim \PP(\cdot | s_t, a_t)}[w_k(s')])
    \big|
    .
\end{align*}
Note that, $w_k$, as the output of the Extended Value Iteration, satisfies the following condition(Lemma~\ref{lemma:EVI}):
\begin{align*}
    |
    r(s_t,a_t)
    +
    \EE_{s' \sim \PP_k(\cdot | s_t, a_t)}[w_k(s')]
    -
    w_k(s_t)
    -
    {\rho}_k
    | \le \epsilon.
\end{align*}
Therefore, we can further bound each term in $I_2$ as follows:
\begin{align*}
    \big| 
    w_k(s_t)
    -
    \EE_{\PP}[w_k(s')]
    \big|
    & =
    \big| 
    w_k(s_t)
    -
    \EE_{\PP_k}[w_k(s')]
    +
    \EE_{\PP_k}[w_k(s')]
    -
    \EE_{\PP}[w_k(s')]
    \big|
    \\
    & \le 
    |
    r(s_t,a_t)
    +
    \EE_{ \PP_k}[w_k(s')]
    -
    w_k(s_t)
    -
    \rho_k
    |
    +
    |
    r(s_t,a_t)
    -
    \rho_k
    |
    \\
    & \qquad 
    +
    \big|
    \EE_{\PP_k}[w_k(s')]
    -
    \EE_{\PP}[w_k(s')]
    \big|
    \\ 
    & \le 
    r_{\max} + r_{\max} 
    +
    \big|
    \EE_{\PP_k}[w_k(s')]
    -
    \EE_{\PP}[w_k(s')]
    \big|
    \\
    & =
    2 r_{\max} 
    +
    \big|
    \la 
    \bphi_{w_k}(s_t,a_t)
    ,
    \btheta_k - \btheta^*
    \ra 
    \big|
    \\
    & \le 
    2 r_{\max} 
    +
    \big\|
    \bphi_{w_k}(s_t,a_t)
    \big\|_{\hbSigma_t^{-1}}
    \cdot
    \big\|
    \btheta_k - \btheta^*
    \big\|_{\hbSigma_t}
    \\
    & \le 
    2 r_{\max} 
    +
    2\hbeta_{t}
    \big\|
    \bphi_{w_k}(s_t,a_t)
    \big\|_{\hbSigma_t^{-1}}.
\end{align*}
Here, the first inequality is due to triangle inequality. The second inequality is due to 1) the reward function( so should the average reward) should lie in $[0, r_{\max}]$ as assumed, and in this paper's setting actually $r_{\max} = 1$. The third inequality is due to Cauchy-Schwartz inequality. The last inequality is due to the assumption $\cE_0$ holds. For the second equality, note that $\EE_{\PP}[w(s')] = \la \bphi_{w}(s'|s_t,a_t), \btheta^* \ra$.

Also, it is clear that $\big|
    \EE_{\PP_k}[w_k(s')]
    -
    \EE_{\PP}[w_k(s')]
    \big| \le D$.
Therefore, term $I_2$ can be bounded as
\begin{align*}
    I_2 
    & \le 
    D
    \sum_{k=0}^{K(T)-1}
    \sum_{t = t_k}^{t_{k+1} - 1} 
    \bigg[ 
    2 r_{\max} 
    +
    \min 
    \Big\{
    D,
    \hbeta_t
    \big\|
    \bphi_{w_k}(s_t,a_t)
    \big\|_{\hbSigma_t^{-1}}
    \Big \}
    \bigg] 
    \\ 
    & =
    2D
    T
    +
    D
    \sum_{k=0}^{K(T)-1}
    \sum_{t = t_k}^{t_{k+1} - 1} 
    \min 
    \Big\{
    D,
    \hbeta_t
    \big\|
    \bphi_{w_k}(s_t,a_t)
    \big\|_{\hbSigma_t^{-1}}
    \Big \}
    \\
    & \le 
    2D
    T
    +
    D
    \sum_{k=0}^{K(T)-1}
    \sum_{t = t_k}^{t_{k+1} - 1} 
    \hbeta_{t} \barsigma_t
    \min 
    \Big\{
    1 ,
    \big\|
    \bphi_{w_k}(s_t,a_t) / \barsigma_t
    \big\|_{\hbSigma_t^{-1}}
    \Big \}
    \\
    & \le 
    2D
    T
    +
    D^2 
    \hbeta_{T} 
    \sum_{k=0}^{K(T)-1}
    \sum_{t = t_k}^{t_{k+1} - 1} 
    \min 
    \Big\{
    1 ,
    \big\|
    \bphi_{w_k}(s_t,a_t) / \barsigma_t
    \big\|_{\hbSigma_t^{-1}}
    \Big \}
    \\ 
    & \le 
    2D
    T
    +
    D^2 
    \hbeta_{T} 
    \sqrt{T}
    \sqrt{
    \sum_{k=0}^{K(T)-1}
    \sum_{t = t_k}^{t_{k+1} - 1} 
    \min 
    \Big\{
    1 ,
    \big\|
    \bphi_{w_k}(s_t,a_t) / \barsigma_t
    \big\|_{\hbSigma_t^{-1}}^2
    \Big \}} 
    \\
    & \le 
    2D
    T
    +
    D^2 
    \hbeta_{T} 
    \sqrt{
    T
    2d 
    \log(1 + T/ \lambda)
    }.
\end{align*}
The second inequality holds because $\hbeta_t \ge \sqrt{d}$ and $\barsigma_t \ge D/\sqrt{d}$.The third inequality holds because $\hbeta_t \le \hbeta_T$ and $\barsigma_t \le D$. The fourth inequality is due to Cauchy-Schwartz inequality. The last inequality is from Lemma~\ref{lemma:det}.

Collecting $I_1$ and $I_2$ gives 
\begin{align*}
    \sum_{k=0}^{K(T)-1}
    \sum_{t = t_k}^{t_{k+1} - 1} 
    [\VV w_k](s_t,a_t)
    \le 
    (D^2/4) \sqrt{2T \log(1/\delta)}
    +
    (K(T) + 1) (D^2/4)
    +
    2DT
    +
    D^2 \hbeta_T 
    \sqrt{
    T
    2d 
    \log(1 + T/ \lambda)
    }
    ,
\end{align*}
given that $\cE_0$ and $\cE_1$ hold. Using big-O notation we have
\begin{align*}
    \sum_{k=0}^{K(T)-1}
    \sum_{t = t_k}^{t_{k+1} - 1} 
    [\VV w_k](s_t,a_t) = \tilde{O} (DT)
    +
    \tilde{O}(D^2d \sqrt{T}). 
\end{align*}
\end{proof}

\subsection{Proof of Lemma~\ref{lemma:boundE}} \label{sec:proof-lemma-boundE}
\begin{proof}[Proof of Lemma~\ref{lemma:boundE}]
Directly unroll the definition of $E_t$:
\begin{align*}
    \sum_{t=1}^{T} E_t
    & =
    \underbrace{
    \sum_{k=0}^{K(T)-1}
    \sum_{t = t_k}^{t_{k+1} - 1} 
    \min \Big \{
    D^2/4
    ,
    \tbeta_t
    \Big\|
    \tbSigma_t^{-1/2}
    \bphi_{w_k^2}(s_t, a_t)
    \Big\|
    \Big \} }_{I_1}
    \\
    & \qquad 
    +
    \underbrace{
    \sum_{k=0}^{K(T)-1}
    \sum_{t = t_k}^{t_{k+1} - 1} 
    \min \Big \{
    D^2/4
    ,
    D 
    \cbeta_t
    \Big \|
    \hbSigma_t^{-1/2} \bphi_{w_k}(s_t, a_t)
    \Big \|
    \Big \}}_{I_2}.
\end{align*}
For term $I_1$,
\begin{align*}
    I_1
    & \le 
    \tbeta_T
    \sum_{k=0}^{K(T)-1}
    \sum_{t = t_k}^{t_{k+1} - 1} 
    \min \Big \{
    1
    ,
    \Big\|
    \tbSigma_t^{-1/2}
    \bphi_{w_k^2}(s_t, a_t)
    \Big\|
    \Big \}
    \\
    & \le 
    \tbeta_T
    \sqrt{T}
    \sqrt{
    \sum_{k=0}^{K(T)-1}
    \sum_{t = t_k}^{t_{k+1} - 1} 
    \min \Big \{
    1
    ,
    \Big\|
    \tbSigma_t^{-1/2}
    \bphi_{w_k^2}(s_t, a_t)
    \Big\|^2
    \Big \}
    }
    \\
    & \le 
    \tbeta_T
    \sqrt{
    2Td
    \log(1 + TD^2/4d\lambda)
    },
\end{align*}
where the first inequality is due to $\tbeta_t \le \tbeta_T$ and $\tbeta_t \ge D^2/4$. The second inequality is due to Cauchy-Schwartz inequality. The third is due to Lemma~\ref{lemma:sumcontext}.

Similarly, for $I_2$, we have
\begin{align*}
    I_2
    & = 
    \sum_{k=0}^{K(T)-1}
    \sum_{t = t_k}^{t_{k+1} - 1} 
    \min \Big \{
    D^2/4
    ,
    D 
    \cbeta_t
    \barsigma_t
    \Big \|
    \hbSigma_t^{-1/2} \bphi_{w_k}(s_t, a_t)
    /
    \barsigma_t
    \Big \|
    \Big \}
    \\
    & \le 
    D^2
    \cbeta_T
    \sum_{k=0}^{K(T)-1}
    \sum_{t = t_k}^{t_{k+1} - 1} 
    \min \Big \{
    1
    ,
    \Big \|
    \hbSigma_t^{-1/2} \bphi_{w_k}(s_t, a_t)
    /
    \barsigma_t
    \Big \|
    \Big \}
    \\
    & \le 
    D^2
    \cbeta_T
    \sqrt{T
    \sum_{k=0}^{K(T)-1}
    \sum_{t = t_k}^{t_{k+1} - 1} 
    \min \Big \{
    1
    ,
    \Big \|
    \hbSigma_t^{-1/2} \bphi_{w_k}(s_t, a_t)
    /
    \barsigma_t
    \Big \|^2
    \Big \}}
    \\
    & \le 
    D^2
    \cbeta_T
    \sqrt{
    2Td
    \log(1 + T/\lambda)
    },
\end{align*}
where the first inequality is due to $\cbeta_t  \barsigma_t \ge D$, $\cbeta_t \le \cbeta_T$ and $\barsigma_t \le D$(all can be verified by the definitions).

To summarize,
\begin{align*}
    \sum_{t=1}^{T} E_t
    & \le 
    \tbeta_T
    \sqrt{
    2Td
    \log(1 + TD^2/4d\lambda)
    }
    +
    D^2
    \cbeta_T
    \sqrt{
    2Td
    \log(1 + T/\lambda)
    }.
\end{align*}
We can also conclude that
\begin{align*}
    \sum_{t=1}^{T} E_t
    & = 
    \tilde{O}(D^2 d^{3/2} \sqrt{T}).
\end{align*}
\end{proof}

\subsection{Proof of Lemma \ref{lemma:n1} } \label{sec:proof-lemma-n1}
\begin{proof}[Proof of Lemma \ref{lemma:n1}]
We have
\begin{align}
    \EE_{\btheta}N_1
    &= \sum_{t=2}^T \cP_{\btheta} (s_t = \state_1)\notag \\
    & = \underbrace{\sum_{t=2}^T \cP_{\btheta} (s_t = \state_1|s_{t-1} = \state_1)\cP_{\btheta}(s_{t-1} = \state_1)}_{I_1} + \underbrace{\sum_{t=2}^T  \cP_{\btheta} (s_t = \state_1,s_{t-1} = \state_0)}_{I_2}.\label{eq:1111}
\end{align}
For $I_1$, since $\cP_{\btheta} (s_t = \state_1|s_{t-1} = \state_1) = 1-\delta$ no matter which action is taken, thus we have
\begin{align}
    I_1 = (1-\delta) \sum_{t=2}^T\cP_{\btheta}(s_{t-1} = \state_1) = (1-\delta) \EE_{\btheta}N_1 - (1-\delta)\cP_{\btheta}(s_T = \state_1). \label{eq:2}
\end{align}
Next we bound $I_2$. We can further decompose $I_2$ as follows.
\begin{align}
    I_2 &= \sum_{t=2}^T \sum_{\ab}\cP_{\btheta}(s_t = \state_1|s_{t-1} = \state_0, a_{t-1} = \ab) \cP_{\btheta}(s_{t-1} = \state_0, a_{t-1} = \ab)\notag \\
    & = \sum_{t=2}^T \sum_{\ab}(\delta + \la \ab, \btheta\ra) \cP_{\btheta}(s_{t-1} = \state_0, a_{t-1} = \ab)\notag\\
    & = \sum_{\ab} (\delta + \la\ab, \btheta \ra)\Big[ \EE_{\btheta} N_0^{\ab} - \cP_{\btheta}(s_T = \state_0, a_T = \ab)\Big]. \label{eq:3}
\end{align}
Substituting \eqref{eq:2} and \eqref{eq:3} into \eqref{eq:1111} and rearranging it, we have
\begin{align}
    &\EE_{\btheta}N_1 \notag \\
    &= \sum_{\ab} (1+ \la \ab, \btheta\ra/\delta) \EE_{\btheta} N_0^{\ab}  - \underbrace{\bigg[\frac{1-\delta}{\delta}\cP_{\btheta}(s_T = \state_1) + \sum_{\ab}(1+ \la \ab, \btheta\ra/\delta)\cP_{\btheta}(s_T = \state_0, a_T = \ab)\bigg]}_{\psi_{\btheta}}\notag \\
    & =  
    \EE_{\btheta}N_0 + \delta^{-1}\sum_{\ab}\la \ab, \btheta\ra \EE_{\btheta}N_0^{\ab} - \psi_{\btheta},\label{eq:4}
\end{align}
which suggests that
\begin{align}
    \EE_{\btheta}N_1 \leq T/2 + \delta^{-1}\sum_{\ab}\la \ab, \btheta\ra \EE_{\btheta}N_0^{\ab}/2.
\end{align}
We now bound $\EE_{\btheta}N_0$. By \eqref{eq:4}, we have
\begin{align}
    \EE_{\btheta}N_1 &\geq \EE_{\btheta}N_0 + \delta^{-1}\sum_{\ab}\la \ab, \btheta\ra \EE_{\btheta}N_0^{\ab} - \psi_{\btheta}\notag \\
    & \geq \EE_{\btheta}N_0 - \frac{\Delta}{\delta}\EE_{\btheta}N_0 - \frac{1-\delta}{\delta}\cP_{\btheta}(s_T = \state_1) - \cP_{\btheta}(s_T = \state_0) - \frac{\Delta}{\delta}\cP_{\btheta}(s_T = \state_0)\notag \\
    & = (1-\Delta/\delta)\EE_{\btheta}N_0 - (1-\delta)/\delta + \frac{1-\Delta}{\delta}\cP_{\btheta}(s_T = \state_0)\notag \\
    & \geq (1-\Delta/\delta)\EE_{\btheta}N_0 - (1-\delta)/\delta\label{eq:5},
\end{align}
where the first inequality holds due to \eqref{eq:4}, the second inequality holds due to the fact that $\la \ab, \btheta\ra \leq \Delta$, the last inequality holds since $\cP_{\btheta}(s_T  = \state_0)>0$. \eqref{eq:5} suggests that
\begin{align}
    \EE_{\btheta}N_0 \leq \frac{T+(1-\delta)/\delta}{2 - \Delta/\delta} \leq \frac{4}{5}T, \notag
\end{align}
where the last inequality holds due to the fact that $2\Delta \leq \delta$ and $(1-\delta)/\delta < T/5$.  
\end{proof}

\subsection{Proof of Lemma \ref{lemma:kl}} \label{sec:proof-lemma-kl}
We need the following lemma:
\begin{lemma}[Lemma 20 in \citet{jaksch2010near}]\label{lemma:basicin}
Suppose $0 \leq \delta' \leq 1/2$ and $\epsilon' \leq 1-2\delta'$, then
\begin{align}
    \delta'\log\frac{\delta'}{\delta'+\epsilon'} + (1-\delta')\log\frac{(1-\delta')}{1-\delta'-\epsilon'} \leq \frac{2(\epsilon')^2}{\delta'}. \notag
\end{align}
\end{lemma}
\begin{proof}[Proof of Lemma \ref{lemma:kl}]
Let $\sbb_t$ denote $\{s_1,\dots,s_t\}$. By the Markovnian property of MDP, we can first decompose the KL divergence as follows:
\begin{align}
    \text{KL}(\cP_{\btheta'}\| \cP_{\btheta}) = \sum_{t=1}^{T-1}\text{KL}\Big[\cP_{\btheta'}(s_{t+1}|\sbb_{t})\Big\| \cP_{\btheta}(s_{t+1}|\sbb_{t})\Big],\notag
\end{align}
where the KL divergence between $\cP_{\btheta'}(s_{t+1}|\sbb_{t}), \cP_{\btheta}(s_{t+1}|\sbb_{t})$ is defined as follows:
\begin{align}
    \text{KL}\Big[\cP_{\btheta'}(s_{t+1}|\sbb_t)\Big\| \cP_{\btheta}(s_{t+1}|\sbb_t)\Big] = \sum_{\sbb_{t+1} \in \cS^{t+1}}\cP_{\btheta'}(\sbb_{t+1})\log\frac{\cP_{\btheta'}(s_{t+1}|\sbb_t)}{\cP_{\btheta}(s_{t+1}|\sbb_t)}.\notag
\end{align}
Now we further bound the above terms as follows:
\begin{align}
     &\sum_{\sbb_{t+1} \in \cS^{t+1}}\cP_{\btheta'}(\sbb_{t+1})\log\frac{\cP_{\btheta'}(s_{t+1}|\sbb_t)}{\cP_{\btheta}(s_{t+1}|\sbb_t)} \notag \\
     &= \sum_{\sbb_t \in \cS^{t}}\cP_{\btheta'}(\sbb_t)\sum_{\state \in \cS} \cP_{\btheta'}(s_{t+1} = \state|\sbb_t)\log\frac{\cP_{\btheta'}(s_{t+1} =\state|\sbb_t)}{\cP_{\btheta}(s_{t+1} = \state|\sbb_t)}\notag \\
     & = \sum_{\sbb_{t-1} \in \cS^{t-1}}\cP_{\btheta'}(\sbb_{t-1})\sum_{\state' \in \cS, \ab \in \cA}\cP_{\btheta'}(s_t = \state', a_t = \ab|\sbb_{t-1})\notag \\
     &\qquad\cdot \sum_{\state \in \cS} \cP_{\btheta'}(s_{t+1} = \state|\sbb^{t-1}, s_t = \state', a_t = \ab)\underbrace{\log\frac{\cP_{\btheta'}(s_{t+1} =\state|\sbb^{t-1}, s_t = \state', a_t = \ab )}{\cP_{\btheta}(s_{t+1} =\state|\sbb^{t-1}, s_t = \state', a_t = \ab )}}_{I_1}\notag ,
\end{align}
When $\state' = \state_1$, $\cP_{\btheta'}(s_{t+1} =\state|\sbb^{t-1}, s_t = \state', a_t = \ab ) = \cP_{\btheta}(s_{t+1} =\state|\sbb^{t-1}, s_t = \state', a_t = \ab )$ for all $\btheta', \btheta$ since the transition probability at $\state_1$ is irrelevant to $\btheta$ due to the MDP we choose. That implies when $\state' = \state_1$, $I_1 = 0$. Therefore, 
\begin{align}
    &\sum_{\sbb_{t+1} \in \cS^{t+1}}\cP_{\btheta'}(\sbb_{t+1})\log\frac{\cP_{\btheta'}(s_{t+1}|\sbb_t)}{\cP_{\btheta}(s_{t+1}|\sbb_t)}\notag \\
    & = \sum_{\sbb^{t-1} \in \cS^{t-1}}\cP_{\btheta'}(\sbb^{t-1})\sum_{\ab}\cP_{\btheta'}(s_t = \state_0, a_t = \ab|\sbb^{t-1})\notag \\
     &\qquad\cdot \sum_{\state \in \cS} \cP_{\btheta'}(s_{t+1} = \state|\sbb^{t-1}, s_t = \state_0, a_t = \ab)\log\frac{\cP_{\btheta'}(s_{t+1} =s|\sbb^{t-1}, s_t = \state_0, a_t = \ab )}{\cP_{\btheta}(s_{t+1} =s|\sbb^{t-1}, s_t = \state_0, a_t = \ab )}\notag \\
     & = \sum_{\ab}\cP_{\btheta'}(s_t = \state_0, a_t = \ab)\notag \\
     &\qquad \cdot\underbrace{\sum_{\state \in \cS} \cP_{\btheta'}(s_{t+1} = s| s_t = \state_0, a_t = \ab)\log\frac{\cP_{\btheta'}(s_{t+1} =\state| s_t = \state_0, a_t = \ab )}{\cP_{\btheta}(s_{t+1} =\state| s_t = \state_0, a_t = \ab )}}_{I_2}\label{eq:kl:1}.
\end{align}
To bound $I_2$, due to the structure of the MDP, we know that $s_{t+1}$ follows the Bernoulli distribution over $\state_0$ and $\state_1$ with probability $1-\delta - \la \ab, \btheta'\ra$ and $\delta + \la \ab, \btheta'\ra$, then we have
\begin{align}
    I_2 &= (1-\la \btheta', \ab\ra - \delta)\log\frac{1-\la \btheta', \ab\ra - \delta}{1-\la \btheta, \ab\ra - \delta} + (\la \btheta', \ab\ra + \delta)\log\frac{\la \btheta', \ab\ra + \delta}{\la \btheta, \ab\ra + \delta}\leq \frac{2\la \btheta' - \btheta,\ab\ra^2}{\la \btheta', \ab\ra+\delta},\label{eq:kl:1.4}
\end{align}
where the inequality holds due to Lemma \ref{lemma:basicin} with $\delta' = \la \btheta', \ab\ra+\delta$ and $\epsilon' = \la \btheta - \btheta', \ab\ra$. It can be verified that
\begin{align}
    \delta' = \la \btheta', \ab\ra+\delta \leq \Delta+\delta \leq 1/2,\label{eq:xxx1}
\end{align}
where the first inequality holds due to the definition of $\btheta'$, the second inequality holds since $\Delta<\delta/2 \leq 1/6$. It can also be verified that
\begin{align}
    \epsilon' = \la \btheta - \btheta',\ab\ra \leq 2\Delta \leq 1-2(\Delta+\delta) \leq 1-2\delta',\label{eq:xxx2}
\end{align}
where the first inequality holds due to the definition of $\btheta', \btheta$, the second inequality holds since $\Delta<\delta/4 \leq 1/12$, the last inequality holds since $\delta' = \la \btheta', \ab\ra+\delta \leq \Delta+\delta$ due to the definition of $\btheta'$. \eqref{eq:xxx1} and \eqref{eq:xxx2} suggest that we can apply Lemma \ref{lemma:basicin} onto \eqref{eq:kl:1.4}. $I_2$ can be further bounded as follows:
\begin{align}
        I_2&\leq \frac{4\la \btheta' - \btheta,\ab\ra^2}{\delta} = \frac{16\Delta^2}{(d-1)^2\delta},\label{eq:kl:2}
\end{align}
where the inequality holds due to \eqref{eq:kl:1.4} and the fact that $\delta+\la \btheta', \ab\ra \geq \delta - \Delta \geq \delta/2$.
Substituting \eqref{eq:kl:2} into~\eqref{eq:kl:1}, taking summation from $t = 1$ to $T-1$, we have
\begin{align}
        \text{KL}(\cP_{\btheta'}\| \cP_{\btheta})& = \sum_{t=1}^{T-1}\sum_{\sbb_{t+1} \in \cS^{t+1}}\cP_{\btheta'}(\sbb_{t+1})\log\frac{\cP_{\btheta'}(s_{t+1}|\sbb_t)}{\cP_{\btheta}(s_{t+1}|\sbb_t)}\notag \\
        & \leq 
        \frac{16\Delta^2}{(d-1)^2\delta}\sum_{t=1}^{T-1}\sum_{\ab}\cP_{\btheta'}(s_t = \state_0, a_t = \ab)\notag \\
        & =         \frac{16\Delta^2}{(d-1)^2\delta}\sum_{t=1}^{T-1}\cP_{\btheta'}(s_t = \state_0)\notag \\
        & \leq \frac{16\Delta^2}{(d-1)^2\delta}\EE_{\btheta'}N_0,\notag
\end{align}
where the last inequality holds due to the definition of $N_0$. 
\end{proof}

\section{Experiments}
In this section, we conduct experiments to empirically study the performance of the proposed algorithm. 

\begin{figure}[H]
    \centering
    \includegraphics[width=0.5\textwidth]{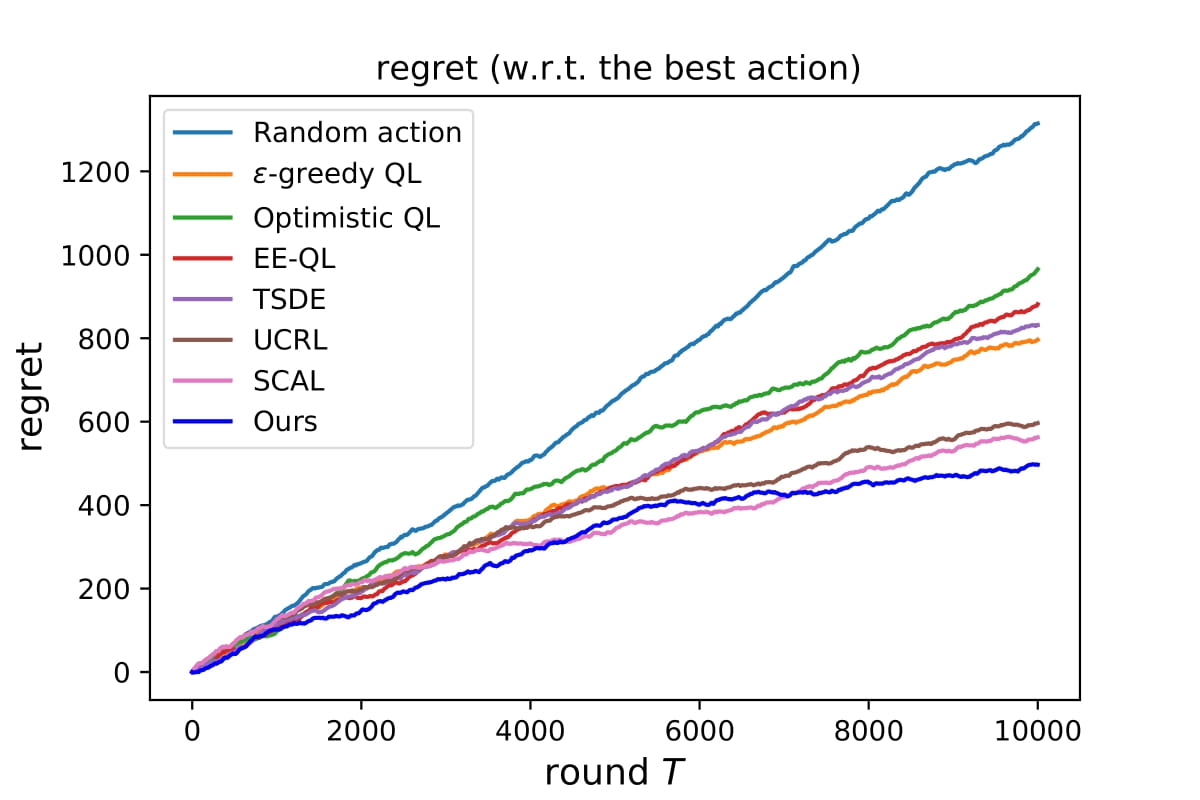}
    \caption{Regret comparison of different algorithms. UCRL2-VTR performs better than the tabular Q-learning by utilizing the given linear structure.}
    \label{fig:fig1}
\end{figure}

 The MDP is constructed as described in Section~\ref{subsec:construction}. We choose $d=8$, and thus $|\cS|=2$ and $|\cA| = 2^{d-1} = 128$. 

We compare the following algorithms:
\begin{enumerate}
    \item Randomly choose an action (Random action).
    \item Q-learning with an $\epsilon$-greedy, uniformly random exploration ($\epsilon$-greedy QL). 
    \item Q-learning with a confidence bonus (Optimistic QL by \citet{wei2020model}).
    \item An Exploration Enhanced Q-learning algorithm (EE-QL by \citet{jafarnia2020model}).
    \item A Thompson sampling-based algorithm (TSDE by \citet{ouyang2017learning}).
    \item A tabular model-based algorithm (UCRL by \citet{jaksch2010near}).
    \item A tabular model-based algorithm that relies on the span of the MDP rather than the diameter (SCAL by \citet{fruit2018efficient}).
    \item Our algorithm with the Hoeffding bonus (Ours).
\end{enumerate}
In our experiments, the parameters of each algorithm are tuned properly. For each algorithm, the experiment is replicated for 10 times and the averaged regret is plotted in Figure~\ref{fig:fig1} for comparison. 
We can see that model-based algorithms (UCRL, SCAL, Ours) are generally better than the model-free ones (Q-learning algorithms and TSDE).
Our proposed algorithm outperforms other model-based algorithms due to utilizing the linear structure of the underlying MDP.


\end{document}